\documentclass[runningheads]{llncs}
\usepackage{graphicx}
\usepackage{algorithm}
\usepackage{float} 
\usepackage[noend]{algpseudocode}

\usepackage{amsmath}
\usepackage{amsfonts}
\usepackage{mathtools}
\usepackage{amssymb}
\usepackage{dsfont} 
\usepackage{xspace} 
\usepackage{subcaption} 
\usepackage{environ}
\usepackage{thmtools} 
\usepackage{thm-restate}

\usepackage{enumitem}
\setlist[enumerate]{wide=-3pt, widest=99,leftmargin=\parindent, labelsep=*} 

\newtheorem{defn}{Definition}
\newcounter{casecount}
\AtBeginEnvironment{proof}{\setcounter{casecount}{0}}

\renewcommand{\epsilon}{\varepsilon}

\newcommand{\X}{\ensuremath{\mathcal{X}}\xspace}
\newcommand{\Y}{\ensuremath{\mathcal{Y}}\xspace}

\newcommand{\F}{\ensuremath{\mathcal{F}}\xspace}

\newcommand{\floor}[1]{\ensuremath{\left\lfloor#1\right\rfloor}}

\newcommand{\E}[1]{\ensuremath{\mathrm{E}\mathord{\left(#1\right)}}}
\newcommand{\Et}[1]{\ensuremath{\mathrm{E}_t\mathord{\left(#1\right)}}}

\newcommand{\MAXIMIN}{\textsc{Maximin}\xspace}
\newcommand{\maximin}{\MAXIMIN}

\newcommand{\Prt}{\mathrm{Pr}_{t}}

\newcommand{\DIAGONAL}{\textsc{Diagonal}\xspace}
\newcommand{\Diagonal}{\DIAGONAL}
\newcommand{\GeDIAGONAL}{\textsc{General-Boundary}\xspace}
\newcommand{\GeDiagonal}{\GeDIAGONAL}

\newcommand{\NASH}{\textsc{Nash}\xspace}
\newcommand{\nash}{\NASH}

\usepackage{hyperref}

\usepackage{xcolor}

\begin{document}
\title{Overcoming Binary Adversarial Optimisation with Competitive Coevolution}
%
\author{Per Kristian Lehre\orcidID{0000-0002-9521-1251} \and
Shishen Lin\orcidID{0009-0008-7405-6564}}
\authorrunning{F. Author et al.}
\institute{University of Birmingham, Birmingham B15 2TT, United Kingdom
\\
\email{\{p.k.lehre,sxl1242\}@cs.bham.ac.uk}}

\maketitle              
\begin{abstract}
Co-evolutionary algorithms (CoEAs), which pair candidate designs with test cases, are frequently used in adversarial optimisation, particularly for binary test-based problems where designs and tests yield binary outcomes. 
The effectiveness of designs is determined by their performance against tests, and the value of tests is based on their ability to identify failing designs, often leading to more sophisticated tests and improved designs. 
However, CoEAs can exhibit complex, sometimes pathological behaviours like disengagement. 
Through runtime analysis, we aim to rigorously analyse whether CoEAs can efficiently solve test-based adversarial optimisation problems in an expected polynomial runtime.
\par This paper carries out the first rigorous runtime analysis of $(1,\lambda)$-CoEA for binary test-based adversarial optimisation problems. 
In particular, we introduce a binary test-based benchmark problem called \Diagonal problem and initiate the first runtime analysis of competitive CoEA on this problem. 
The mathematical analysis shows that the $(1,\lambda)$-CoEA can efficiently find an $\varepsilon$ approximation to the optimal solution of the \Diagonal problem, i.e. in expected polynomial runtime assuming sufficiently low mutation rates and large offspring population size.
On the other hand, the standard $(1,\lambda)$-EA fails to find 
an $\varepsilon$ approximation to the optimal solution of the \Diagonal problem in polynomial runtime. 
This illustrates the potential of coevolution for solving binary adversarial optimisation problems.

\end{abstract}

\section{Introduction}

CoEAs are a class of algorithms that have been applied in various game-theoretic and strategic optimisation scenarios. 
There are two main types of CoEAs: cooperative and competitive
CoEAs. 
Competitive CoEAs can be applied to problems that involve adversaries, such as \maximin optimisation problems. 
With the widespread use and development of GANs
\cite{goodfellow_generative_nodate}, there are recent successful
applications of competitive CoEAs for GANs \cite{al-dujaili_towards_2018,toutouh_spatial_2019} and co-evolutionary learning \cite{mitchell2006coevolutionary}. 
Competitive CoEAs share similarities with neural network-based adversarial models but require less information, e.g., a gradient. However, despite their potential, the application of CoEAs is challenging. 
One of the main difficulties is that these algorithms often 
exhibit pathological behaviours, such as cyclic behaviours, 
disengagement, and over-specialisation \cite{rozenberg_coevolutionary_2012,wiegand2004analysis}. 
More precisely, cyclic behaviour means that the solution $A$ dominates $B$, $B$ dominates $C$, but $C$ can dominate $A$. This leads to the problem of the algorithm forgetting the optimal solution previously found; for example, RLS-PD suffers from an evolutionary forgetting issue for finding \nash Equilibrium~\cite{MarioRLSBi23}. Disengagement and over-specialisation mean one of the population is too strong, and the other barely learns or optimises from coevolution.  These challenges limit the widespread use of CoEAs.

There is a growing interest in optimisation problems that involve one or more adversaries. These problems include robust optimisation or designing game-playing strategies. We focus on a special case of adversarial problems called test-based optimisation problems \cite{jaskowski2011algorithms}. Test-based optimisation is an important class of optimisation problems where individuals in a population of designs are evaluated against test cases, which co-evolve with the designs \cite{jong_ideal_2004}. For example, in supervised learning,  we consider model parameters as solutions and training data as test cases \cite{tibshirani_elements_nodate}. Researchers also apply reinforcement learning methods in board games, including Go \cite{silver_mastering_2016} and Stratego \cite{perolat_mastering_2022}, which consider the opponents as test cases in their self-plays.
Gradient-based methods have been used to tackle these problems when the ``payoff" function is differentiable \cite{ruder_overview_2017}. However, in many real-world scenarios, the payoff function is not differentiable, for example, when the strategy space is discrete. 
In these cases, CoEAs have been suggested to be a promising approach \cite{rozenberg_coevolutionary_2012}. A notable early 
success in this field was the work by Hillis \cite{hillis_co-evolving_1990}, who used competitive CoEAs to optimise 
sorting networks and their corresponding test cases.
Other early examples of the successful application of CoEAs on test-based problems can be found in works by \cite{axelrod1987evolution} and \cite{lindgren1992evolutionary}. De Jong et al. explored test-based problems in the context of either coevolution \cite{jong_ideal_2004} or multi-objective optimisation \cite{de_jong_multiobjective_2023} from empirical studies. There is a gap in the theoretical understanding of CoEAs on test-based problems.

Hillis \cite{hillis_co-evolving_1990} showed empirically that there is a significant improvement in sorting networks via competitive CoEAs compared with normal EAs. But it is still unclear why competitive CoEAs lead to a better design than traditional EAs \cite{rozenberg_coevolutionary_2012,rosin1997coevolutionary}.
We would like to understand how competitive CoEAs work on test-based optimisation problems from the simplest example.
We formalise a general problem class, which includes Hillis' co-evolutionary 
approach on sorting networks as follows: consider a function $g:\X \times \Y 
\rightarrow \{0,1\}$, where $\X$ is a set of designs and $\Y$ is a set of test 
cases. We define $g(x,y)=1$ if and only if design $x$ passes test case $y$. Our 
optimisation problem can be defined as follows: 
$arg \max_{x \in \X} \min_{y \in \Y} g(x,y)$.
In other words, the {\sc{Maximin}} Optimisation is to find $(x^*,y^*) \in \X \times \Y$ such that
\begin{align*}
    \text{ $\text{for all $(x,y) \in \X \times \Y$, } g(x,y^*)\leq g(x^*,y^*) \leq g(x^*,y)$}.  
\end{align*}
 So here come the following research questions: 
 \begin{enumerate}
      \item Under what circumstances can a competitive CoEA obtain an optimal solution in polynomial expected time? 

    \item  How does the runtime depend on the problem $(g)$ and on the algorithm?
 \end{enumerate}

\par To understand the questions above, we proceed by using runtime analysis. Runtime analysis of traditional
evolutionary algorithms considers the time complexity of a given
randomised algorithm. It provides either lower or upper bounds for the
number of fitness function evaluations (called runtime)
to understand the performance of given algorithms
\cite{doerr_theory_2020}. Runtime analysis can identify relationships
between algorithmic parameters and problem characteristics that
determine the efficiency of evolutionary algorithms. There is limited literature about runtime analysis of CoEAs apart from \cite{jansen_cooperative_2004,lehre_runtime_2022,MarioRLSBi23}.
We want to develop more runtime analysis for CoEAs
and expect the insights from runtime analysis of CoEAs will improve the design of CoEAs \cite{rozenberg_coevolutionary_2012}.

\subsubsection*{Our Contributions}
We first introduce a formulation of a binary test-based adversarial optimisation problem, the \Diagonal problem.
We prove that the traditional $(1,\lambda)$-EA cannot solve the \Diagonal Game in expected polynomial runtime. 
However, we rigorously show for the first time that, 
with the help of coevolution, 
(1,$\lambda$)-CoEA can solve \Diagonal problems in expected polynomial runtime under certain settings by using a two-phase analysis and order statistics tools.
This suggests the promising potential of coevolution for solving binary adversarial optimisation problems.


\section{Preliminaries}
For a filtration $\F_t$, we write $\Et{\cdot}:=\E{\cdot|\mathcal{F}_{t}}$ and $\Prt(\cdot):=\Pr(\cdot|\mathcal{F}_{t})$.
Denote the $1$-norm as $|z|_1=\sum_{i=1}^n z_{i}$ for $z \in \{0,1\}^n$.
Denote $X_t=|x_t|_1 \in [n]\cup \{0\}$ for $x_t \in \{0,1\}^n$ and  $Y_t=|y_t|_1\in [n]\cup \{0\}$ for $y_t \in \{0,1\}^n$ for any $n \in \mathbb{N}$.
We focus on the search space $\X \times \Y = \{0,1\}^n \times \{0,1\}^n$ for any $n \in \mathbb{N}$.
We consider the filtration $(\mathcal{F}_{t})_{t\geq 0}$  including the information of $(X_0,Y_0), \dots ,(X_t,Y_t)$ in this paper.
 For any $\varepsilon \in [0,1]$ and any problem with a unique optimum $(x^*,y^*)$, we say that algorithm $A$ finds an $\varepsilon$-approximation to the optimum $(x^*,y^*)$ in iteration $T \in \mathbb{N}$ if $||x^*|_1-|x_T|_1|+||y^*|_1-|y_T|_1|<\varepsilon n$ where $(x_T,y_T)$ is the search point of $A$ at iteration $T$.
``With high probability'' is abbreviated  to ``w.h.p.''.






\subsection{\Diagonal Games}
In order to model the binary test-based optimisation problem inspired by Hillis's method of sorting 
networks \cite{hillis_co-evolving_1990}, a payoff function with $\X \times \Y $ as input and $\{0,1\}$ as output is introduced as follows. 
Here $\X = \{0,1\}^n $ is the solution space of a set of designs for sorting networks and $\Y = \{0,1\}^n$ is the solution space of a set of test cases. 
We continue to consider a function $g:\X \times \Y 
\rightarrow \{0,1\}$. 
We define $g(x,y)=1$ if and only if design $x$ passes test case $y$. Our optimisation problem can be defined in terms of \maximin optimisation:
\textit{find $(x^*,y^*) \in \X \times \Y $ such that } 
\begin{align*}
    \text{ $\text{for all $(x,y) \in \X \times \Y$, } g(x,y^*)\leq g(x^*,y^*) \leq g(x^*,y)$}.  
\end{align*}
We want to start our analysis with simple problems with clear structures 
so we introduce a constraint function $c$ as follows, which splits the search space into several parts.
\begin{defn}(Generalised boundary test-based problem) Given a constraint function $c(z):\mathbb{R}\rightarrow \mathbb{R}$, a generalised boundary function is called \GeDiagonal $g_{c}:\X \times \Y \rightarrow \{0,1\}$, where $\X=\{0,1\}^n$ and $\Y=\{0,1\}^n$, if 
\begin{align*}
    g_{c}(x,y)=\begin{cases} 
      1 & |y|_1\leq c(|x|_1)  \\
      0 & \text{otherwise} .
   \end{cases}
\end{align*}
In our case, we start with a linear constraint function $c(|x|_1)=|x|_1$.
\end{defn}

 \begin{defn} For $\X=\{0,1\}^n$ and $\Y=\{0,1\}^n$, the payoff function $\Diagonal:\X \times \Y \rightarrow \{0,1\}$ is 
\begin{align*}
    \Diagonal(x,y)
    := \begin{cases} 
      1 & |y|_1\leq |x|_1   \\
      0 & \text{otherwise}
   \end{cases}.
\end{align*}
 \end{defn}


\begin{figure}[htpb] 
        \centering
         \includegraphics[width=0.33\textwidth]{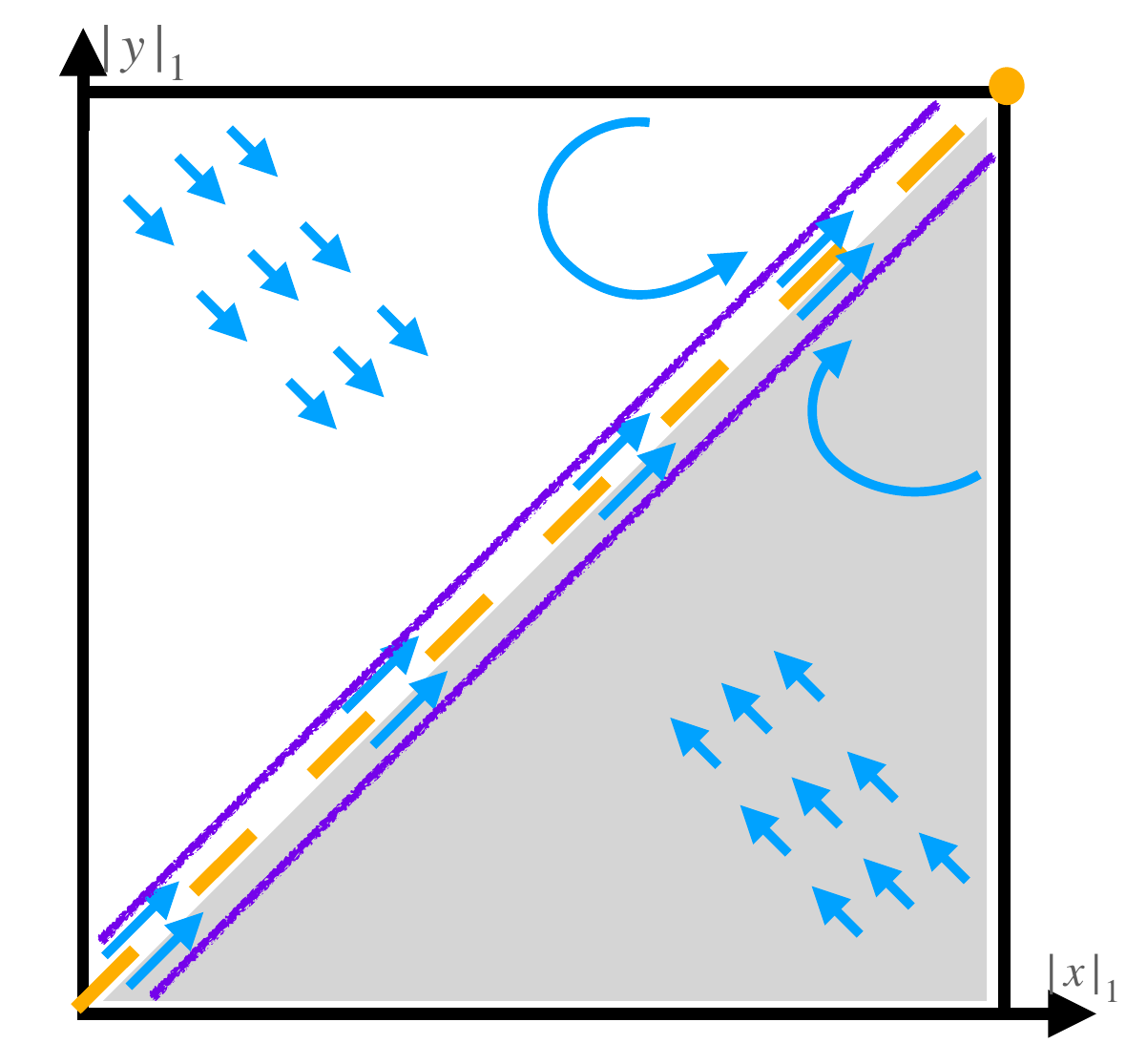}
     \caption{Example of \Diagonal problem. The horizontal axis represents the number of $1$-bits in the designs $x$, and the vertical axis represents the number of $1$-bits in the test cases $y$. The grey area represents search points of payoff $1$, and the rest represents search points with payoff $0$. }
    \label{fig:DA1commalambda1}
\end{figure}

We also use $g(x,y)$ to denote \Diagonal. 
Notice that $1$ means the design $x$ passes the test cases $y \in \Y$. 
In this simple \Diagonal game, 
$y_{n}$ with $|y_{n}|_1=n$ represents the most difficult test case, and $x_{n}$ with $|x_{n}|_1=n$ is the only solution that can pass $y_n$ (i.e. $g(x_n,y_n)=1$).
Thus,  $(1^n,1^n)$ is the \maximin optimum in this case. 
In this optimum, neither the design nor the test case is willing to deviate from affecting their payoff $g(x,y)$ anymore. This exactly coincides with the definition of Nash equilibrium. 
This paper aims to explore whether the CoEAs can find such an optimal solution efficiently. 

\subsection{Drift Analysis Toolbox}
Before our analysis, we introduce the Negative Drift Theorem \cite{oliveto_erratum_2012,rowe_choice_2012}, which will be used to prove the inefficiency of algorithms.

\begin{theorem}[Negative Drift Theorem \cite{oliveto_erratum_2012,rowe_choice_2012,doerr2019theory}] 
\label{thm:NegativeDrift}
For constants $a,b,\delta, \eta, r>0$, with $a<b$, there exist $c>0$, $n_{0}\in \mathbb{N}$ such that the following holds for all $n\geq n_{0}$. Suppose $(X_{t})_{t\geq 0}$ is a sequence of random variables with a finite state space $S\subset \mathbb{R}_{0}^{+}$ and with associated filtration $\mathcal{F}_{t}$. Assume $X_{0}\geq bn$, and let $T_{a}:=\min \{t\geq 0 \mid X_{t}\leq an\}$ be the hitting time of $\mathcal{S}\cap [0,an]$. Assume further that for all $s\in \mathcal{S}$ with $s>an$, for all $j\in \mathbb{N}_{0}$, and for all $t\geq 0$, the following conditions hold:
\begin{itemize}
    \item[(1)] $\E{X_{t}-X_{t+1} \mid \mathcal{F}_{t}, X_{t}=s  } \leq - \delta$
    
    \item[(2)] $\Pr[|X_{t}-X_{t+1} |\geq j  \mid  \mathcal{F}_{t}, X_{t}=s ] \leq \frac{r}{(1+\eta)^j}$

\end{itemize}
Then, 
$
    \Pr[T_{a} \leq e^{cn}] \leq e^{-cn}.
$

\end{theorem}

Before our analysis, we introduce the Additive Drift Theorem \cite{he_drift_2001,doerr2019theory}, which will be used to provide the bounds for the runtime of algorithms.

\begin{theorem}[Additive Drift Theorem \cite{he_drift_2001,doerr2019theory}]
\label{thm:additivedrift}
Let $(X_t)_{t\geq 0}$ be a sequence of non-negative random variables with a finite state space $\mathcal{S}\subseteq \mathbb{R}_0^{+}$ such that $0 \in \mathcal{S}$.
Let $T:=\inf \{ t\geq 0 \mid X_t=0\}$.

\begin{enumerate}
    \item[(1)] If there exists $\delta>0$ such that for all $s \in \mathcal{S} \setminus \{0\}$ and for all $t\geq 0$, $\E{X_t-X_{t+1}\mid X_t=s} \geq \delta$, then $\E{T} \leq \E{X_0} /{\delta}$.
     \item[(2)] If there exists $\delta>0$ such that for all $s \in \mathcal{S} \setminus \{0\}$ and for all $t\geq 0$, $\E{X_t-X_{t+1}\mid X_t=s} \leq \delta$, then $\E{T} \geq \E{X_0} /{\delta}$.
\end{enumerate}

\end{theorem}

\section{Traditional Evolutionary Algorithm cannot Solve Diagonal efficiently}
In this section, we would like to explore whether traditional $(1,\lambda)$-EA can efficiently solve problems with only binary fitness.
For a fair comparison between traditional evolutionary and coevolutionary algorithms,
we chose $(1,\lambda)$-EA as the closest traditional evolutionary algorithm to the coevolutionary algorithm studied in this paper.

\begin{algorithm}
\caption{\((1,\lambda)\)-EA \cite{jagerskupper2007plus}}
  \label{alg:onecommalambda}
  \begin{algorithmic}[1]
    \Require Search spaces $\mathcal{X}$.
    \Require Mutation $\mathcal{D}_{\operatorname{mut}}^t:\Omega \times \{0,1\}^n\rightarrow \{0,1\}^n$.
    \Require Payoff function $f:\mathcal{X} \rightarrow \mathbb{R}$.  
    \State Set $t := 1$ and choose $x_t \in \X$ uniformly at random.
    \Loop \quad until the termination criteria met
        \State Set $t := t + 1$
        \State Let $y_{t,1} := \mathcal{D}(\omega_t,x_{t-1}), \ldots, y_{t,\lambda} :=  \mathcal{D}(\omega_t, x_{t-1})$.
        \State Choose $y_t \in \{y_{t,1}, \ldots, y_{t,\lambda}\}$  among all elements with the largest $f$-value.
        \State \((1,\lambda)\)-EA: Set $x_t := y_t$.
        \State Go to 2.
    \EndLoop
\end{algorithmic}
\end{algorithm}

\par Algorithm~\ref{alg:onecommalambda} samples $x$ uniformly at random. 
We define the same mutation operator $\mathcal{D}_{\operatorname{mut}}^t$ for $x$.
$\Omega$ is the sample space and $\omega_t \in \Omega$ means that the algorithm performs bit-wise mutation for each bit in the bit-string with probability $\chi /n$ where $\chi \in (0,1]$ in iteration $t$ where $x$ is of length $n $\footnote{We consider $\chi \in (0,1]$ including the default choice $\chi=1$ used in \cite{rowe_choice_2012}.}.
Next, we evaluate each individual by taking the fitness of $i$-th offspring.
Then, until the termination criteria are met, only the bit-wise mutation operator mutates $x$.
After that, Algorithm~\ref{alg:onecommalambda} selects the individual with the best fitness. 
If there is a tie in line $5$ of Algorithm~\ref{alg:onecommalambda}, then we consider choosing among all the individuals of the highest fitness uniformly at random. Next, we prove the following theorem. 

\begin{restatable}{theorem}{SecFourMainZero}
\label{thm:Main0} 
Given any $\varepsilon \in (0,1/4)$, $n \in \mathbb{N}$ and the function $f:\{0,1\}^{2n} \rightarrow\{0,1\}$ s.t.
$z=(x,y)$ where $x,y \in \{0,1\}^{n}$ and $f(z)=\Diagonal(x,y)$,
the runtime of $(1,\lambda)$-EA with $\lambda=\text{poly}(n)$ and constant $\chi \in (0,1]$ on finding an $\varepsilon$-approximation to the optimum (i.e. \nash Equilibrium) of $f$ is at least $e^{\Omega(n)}$ with probability $1- e^{-\Omega(n)}$.
\end{restatable}

\par From our analysis of $(1,\lambda)$-EA on \Diagonal, Theorem~\ref{thm:Main0} shows that for any constant $\varepsilon \in (0,1/4)$, 
standard  $(1,\lambda)$-EA cannot find any $\varepsilon$-approximation of the \maximin optimum of \Diagonal efficiently. 
Moreover, for expected runtime $T_{\varepsilon}$, if we apply Markov's inequality:    
    \begin{align*}
         E[T_{\varepsilon}] \geq e^{cn} \Pr(T_{\varepsilon}\geq e^{cn}) = e^{cn}(1-e^{-cn}) = e^{\Omega(n)}.
    \end{align*}
\Diagonal only consists of binary values, resulting in a very hard and flat fitness landscape in the search space.
It leads to a random walk of the search point on the search space, which consists of fitness $1$ (i.e. in a grey area of Figure~\ref{fig:DA1commalambda1}).
A further insight is that traditional EAs (e.g. $(1,\lambda)$-EA, $(1+1)$ EA) cannot cope well with the interaction between $x$ and $y$  
since these algorithms might only favour the highest (or lowest) fitness other than \nash equilibrium in \Diagonal \footnote{The proof of Theorem~\ref{thm:Main0} can be generalised to a broader class of traditional EAs}.
The approach of considering a pair of individuals from two distinct populations ($x \in \X$ and $y \in \Y$) as a single entity within one population (represented as $z:=(x,y)$) fails to capture the complexity of interactions between these two populations.
So is there any alternative
evolutionary approach that could help us to resolve this issue?

\section{Competitive Coevolution Solves Diagonal Efficiently}
After our analysis, we know that $(1,\lambda)$-EA
cannot solve the \Diagonal Games efficiently. 
We are wondering whether a competitive CoEA exists that can solve this problem by changing the selection mechanism or increasing the size of offspring. 
The key is to properly capture the complex interaction between populations so the algorithm can find \nash equilibrium efficiently.
So, we extend the traditional evolutionary algorithm in the context of competitive CoEAs and consider the $(1,\lambda)$-variant of CoEAs.
Then, we prove $(1,\lambda)$-CoEA solves the \Diagonal problem in expected polynomial runtime.
\begin{algorithm}
  \caption{$(1,\lambda)$-CoEA (Alternating Update)}
    \label{alg:ocoea2}
  \begin{algorithmic}[1]
  	\Require Search spaces $\mathcal{X},\mathcal{Y}$.
    \Require Mutation $\mathcal{D}_{\operatorname{mut}}^t(x):\Omega \rightarrow \{0,1\}^n$ for all $x \in \{0,1\}^n$
    \Require Payoff function $g:\mathcal{X} \times \mathcal{Y} \rightarrow \mathbb{R}$.  
    \State Sample $x\sim\mathrm{Unif}(\mathcal{X})$; Sample $y\sim\mathrm{Unif}(\mathcal{Y})$.
    \For{$t \in \{1, 2,\dots\}$}{
        \If{ $t \bmod 2 = 0$ }
            \For{$i=1$ to $\lambda$}
              \State $x^{i} \sim \mathcal{D}_{\operatorname{mut}}^t(x)$; 
            \EndFor
            \State $x:=\arg \max_{i \in [\lambda]}g(x^{i},y)$; 
        \Else \For{$i=1$ to $\lambda$}
              \State $y^{i} \sim \mathcal{D}_{\operatorname{mut}}^t(y)$; 
            \EndFor
            \State $y:=\arg \max_{i \in [\lambda]}-g(x,y^{i})$;
        \EndIf
    \EndFor}
  \end{algorithmic}
\end{algorithm}

\par Algorithm~\ref{alg:ocoea2} samples both design $x$ and test case $y$ uniformly at random. 
We define the same mutation operator $\mathcal{D}_{\operatorname{mut}}^t$ for both $x$ and $y$. 
$\Omega$ is the sample space and $\mathcal{D}_{\operatorname{mut}}^t(x)$ means that for any $x\in \{0,1\}^n$, the algorithm performs bit-wise mutation for each bit in the bit-string with probability $\chi / n$ for $\chi \in (0,n)$ in iteration $t$. 
Then, instead of using pairwise dominance, we evaluate each design by taking the payoff of $i$-th offspring against the parent opponent and evaluate each test case by taking the payoff of it against the parent design for $\lambda$ offspring.
Then, until the termination criteria are met, only the bit-wise mutation operator mutates either $x$ or $y$ in an alternating manner. 
After that, Algorithm~\ref{alg:ocoea2} selects the pair of the design and the test case of the best fitness. In the following analysis, we define $f(x^{i}):=g(x^{i},y)$ and $h(y^{i}):=-g(x,y^{i})$ where $x, y, x^i, y^i$ are defined in Algorithm~\ref{alg:ocoea2}. 
In this paper, in order to archive polynomial runtime for Algorithm~\ref{alg:ocoea2}, 
we restrict $\lambda \in \text{poly}(n)$. 

\subsection{Characteristic Lemma for Alternating Update}
In this section, without loss of generality, we write the $\lambda$ offspring at generation $t\in \mathbb{N}$ in descending order in their 1-norms: $|x_t^{(1)}|_1 \geq |x_t^{(2)}|_1 \geq \cdots \geq |x_t^{(\lambda)}|_1$ and $|y_t^{(1)}|_1 \geq |y_t^{(2)}|_1 \geq \cdots \geq |y_t^{(\lambda)}|_1$.
We also use $X_t^{(i)}$ and $Y_t^{(i)}$ to denote $|x_t^{(i)}|_1$ and $|y_t^{(i)}|_1$ respectively. $X_t:=|x_t|_1$ and $Y_t=|y_t|_1$ where $x_t,y_t \in \{0,1\}^n$ are current search point at iteration $t\in \mathbb{N}$.

\begin{restatable}{lem}{SecFivelemOne}
\label{lem:AvgCharacteristic1}
Consider the fitness of $x$-bitstring denoted by $f$ and the fitness of $y$-bitstring denoted by $h$ in Algorithm~\ref{alg:ocoea2}, 
\begin{itemize}
    \item[(1)] If $|x^{(1)}|_1\geq |x^{(2)}|_1$, then $f(x^{(1)})\geq f(x^{(2)})$.
    \item[(2)] If $|y^{(1)}|_1\geq |y^{(2)}|_1$, then $h(y^{(1)})\geq h(y^{(2)})$.
\end{itemize}
\end{restatable}

    By Lemma~\ref{lem:AvgCharacteristic1}, if $\lambda$ offspring all have the same 1-norms, then they have the same fitness. The algorithm will conduct a random walk around the search space. The algorithm makes actual progress based on argmax selection mechanism when ``crossing the diagonal". 
    Next, we rigorously define it.

\begin{defn}(Cross the diagonal) 
\label{def:cross}
Given Algorithm~\ref{alg:ocoea2} applied to \Diagonal, and for current search point $(x_t,y_t)\in \{0,1\}^n\times \{0,1\}^n$, assume that we have the $\lambda$ offspring at iteration $t\in \mathbb{N}$ in descending order in their 1-norms.
We say that $\lambda$ offspring cross the diagonal \textbf{horizontally} at iteration $t$, 
if there exist some $k$ such that $|x_t^{k}|_1\ge |y_t|_1$ with $|y_t|_1>|x_t|_1$; 
We say that $\lambda$ offspring cross the diagonal \textbf{vertically} at iteration $t$ if there exist some $\ell$ such that $|y_t^{\ell}|_1>|x_t|_1$ with $|x_t|_1\geq |y_t|_1$. We say $\lambda$ offspring cross the diagonal if either occurs.
\end{defn}

Definition~\ref{def:cross} means that when crossing the diagonal, the fitness of either $x$-bitstring or $y$-bitstring strictly improves.
Next, we introduce the concept of a $c$-tube (the purple strip as presented in Fig.~\ref{fig:DA1commalambda1}). 
\begin{defn}($c$-tube) 
We call $ C = \{(x,y)\in \mathcal{X} \times \mathcal{Y} \mid  ||x|_1-|y|_1|<c\}$ $c$-tube. We say a current search point $(x,y)\in\mathcal{X} \times \mathcal{Y}$ lies outside the $c$-tube if $(x,y) \notin C$.
\end{defn}



\subsection{Phase 1}
Algorithm~\ref{alg:ocoea2} updates the search point in an alternating manner. From the characteristic lemma and definition of crossing the diagonal, we know when $Y_t-X_t>0$, Algorithm~\ref{alg:ocoea2} makes progress on searching the optimum by updating $X_t$ to let it cross the diagonal and vice versa. 
In the analysis, we want to avoid the case when $Y_t-X_t>0$, but Algorithm~\ref{alg:ocoea2} updates $Y_t$ instead of $X_t$ in $c$-tube. This inspires the following definition. We define a successful cycle formally.
\begin{defn}
\label{def:phase1}
Given $c > 0$, then for all $t\in \mathbb{N}$, we have a successful cycle in iteration $2t$ with respect to $c$ which consists of two consecutive steps  if $X_{2t}+c>Y_{2t}> X_{2t}$, $Y_{2t+1}+c>X_{2t+1} \geq Y_{2t+1}$ and $X_{2t+2}+c>Y_{2t+2}>X_{2t+2}$. We have a successful cycle in iteration $2t+1$ with respect to $c$ which consists of two consecutive steps if $Y_{2t+1}+c>X_{2t+1}\geq Y_{2t+1}$, $X_{2t+2}+c>Y_{2t+1} \geq X_{2t+2}$ and $X_{2t+3}+c>Y_{2t+3}\geq Y_{2t+3}$.
\end{defn}
A successful cycle implies that the algorithm crosses the diagonal twice without leaving the $c$-tube. We show that the search point will move along the diagonal towards the optimum. We use $H_t:=2n-(X_t+Y_t)$ as the potential function to show there is a positive drift towards the optimum $(X_t,Y_t)=(n,n)$ although with a small probability of escaping from the $c$-tube.

In the following analysis, we consider a deterministic initialisation to simplify the analysis. The following analysis sufficiently covers the core principle of how Algorithm~\ref{alg:ocoea2} works via competitive coevolution. We consider an initialisation at $(0^n,0^n)$ which is the farthest search point with respect to the optimum $(n,n)$ in terms of Hamming distance. 
First, we present our main lemma in this subsection.

\begin{restatable}{lem}{SecFivePhaseOneMainOne}
\label{lem:phase1} Assume that Algorithm~\ref{alg:ocoea2} is initialised with  $(X_0,Y_0)=(0,0)$. For all $t \in \mathbb{N}$, we define $D_t:=|X_t-Y_t|$ and $T:=\inf \{t>0 \mid H_t:=2n-X_t-Y_t <\varepsilon n\}$ and for all $\tau \in \mathbb{N}$, let event $E_{\tau}$ denote that the algorithm has $\tau$ consecutive successful cycles with respect to $c$ or $\min_{t \in [2\tau]}H_{t}<\varepsilon n$. If $~\exists c>0, p_c,p_e \in [0,1]$ s.t. for all $t \in [1,T/2)$,
\begin{itemize}
    \item[(1)] $\Pr \left(X_{2t+1} \geq Y_{2t+1}\mid X_{2t}<Y_{2t} \right)\geq 1-p_c;$ 
    \item[(2)] $\Pr \left(X_{2t+2}<Y_{2t+2} \mid X_{2t+1} \geq Y_{2t+1} \right) \geq 1-p_c;$
    \item[(3a)] $\Pr \left(D_{2t+1}>c \mid D_{2t}<c \right) \leq p_e$;

    \item[(3b)] $\Pr \left(D_{2t+2}>c \mid D_{2t+1}<c \right) \leq p_e$, 
\end{itemize}
 then $\Pr \left(E_{\tau} \right) \geq 1-2\tau(p_c+p_e)$ for any constant $\varepsilon \in (0,1)$.
\end{restatable}

The conditions $(1),(2)$ means that the search point crosses the diagonal at iteration either $2t$ or $2t+1$ with probability at least $1-p_c$. The condition $(3)$ means that the search point escapes from the $c$-tube at iteration $t+1$ with probability at most $p_e$ for all $t \in [1,T)$. So Lemma~\ref{lem:phase1} gives a lower bound for the probability the algorithm has $\tau$ consecutive successful cycles.

We will use Lemma~\ref{lem:phase1} to wrap up all the arguments with $\tau=\Omega(n)$ and $\lambda=n^{\Omega(1)}$ later. We proceed with Phase 2 by assuming a successful cycle always exists in the analysis. In other words, when $Y_t>X_t$, a successful cycle guarantees that we update $X_t$ and when $Y_t\leq X_t$, a successful cycle guarantees that we update $Y_t$. We will compute $p_c$ and $p_e$ in Phase 2.

\subsection{Phase 2}
Let us define $D_t=|X_t-Y_t|$ to be the distance away from the diagonal. After the search point crosses the diagonal, we start Phase 2. We divide Phase 2 into three sections. Firstly, we show that given a search point stays within some $c$-tube, the $\lambda$ offspring cross the diagonal with probability $1-2\left(\frac{1}{\lambda}\right)^{\frac{1}{2e^{\chi}}}$. Meanwhile, the search point escapes from the $c$-tube with prob. bounded by $\frac{1}{\lambda^{O(1)}}$. 
Secondly, we show that we always have a positive drift when the search point crosses the diagonal. 
Finally, we wrap up everything using a restart argument, which accounts for the failed generations.

\subsubsection{Phase 2.1}
\par Firstly, we need some lemmas to formulate the most likely scenarios when $\lambda$ offspring are produced.

\begin{restatable}{lem}{SecFivePhaseTwoOnelemOne}\cite{cameron2007notes}
 \label{lem:binomial}Given a binomial random variable $Z\sim Bin(n,p)$,  $\Pr\left(\text{$Z$ is even} \right)=\frac{1}{2}+\frac{1}{2} \cdot(1-2p)^n$.
    
\end{restatable}

We will use Lemma~\ref{lem:binomial} to show the following Lemma~\ref{lem:crossD0}. So we can see from Lemma~\ref{lem:crossD0}, that we need sufficiently large offspring size to avoid the case that $X_t^{(1)}$ coincides with $X_t^{(\lambda)}$ and guarantee that there is some offspring identical to the parents pair with high probability.

    


\begin{restatable}{lem}{SecFivePhaseTwoOnelemTwo}\label{lem:crossD0} Given problem size, offspring size and iteration $n, \lambda, t \in \mathbb{N}$ and mutation rate $\chi=O(1)$, if $t$ is even, then
\begin{align*}
    \Pr \left( \exists k \in [\lambda] , x_t^{k}=x_t\right) 
    \geq 1-e^{-\Omega(\lambda)}, 
    \Pr \left( \max_{i \in [\lambda]}X_t^i >\min_{i \in [\lambda]}X_t^i\right) 
    \geq 1-e^{-\Omega(\lambda)}
\end{align*}
If $t$ is odd, then
\begin{align*}
    \Pr \left( \exists \ell \in [\lambda] , y_t^{\ell}=y_t\right) 
    \geq 1-e^{-\Omega(\lambda)},  
    \Pr \left( \max_{i \in [\lambda]}Y_t^i >\min_{i \in [\lambda]}Y_t^i\right) 
    \geq 1-e^{-\Omega(\lambda)}.
\end{align*}
\end{restatable}

Next, we provide some useful concentration inequality, which can be used to show how much deviation is made by each $i$-th offspring in the number of $1$-bits. To obtain the concentration, we proceed by using Moment Generating Functions (MGFs).

\begin{restatable}{lem}{SecFivePhaseTwoOnelemThree}\label{lem:MGF1} 
For $n \in \mathbb{N}$ and any $s \in [n]\cup \{0\}$, 
we define $U = V_1-V_2 $ where $V_1 \sim Bin(n-s,\chi/n)$ and $V_2 \sim Bin(s, \chi/n)$ are independent,  with $\chi <n$. The MGF (moment generating function) $M_U(\eta) \leq \exp \left(\chi(e^{\eta}-1)\right) $ for any $\eta>0$.
\end{restatable}

\begin{restatable}{lem}{SecFivePhaseTwoOnelemFour}
\label{lem:MaxPoisson} 
With the same setting as Lemma~\ref{lem:MGF1}, 
for any $s \geq 0$ and $\lambda \geq 1$, $\Pr \left(U  \geq s \right) \leq e^{-\chi} \lambda ^{\chi} e^{- s\ln \ln \lambda }$. Furthermore, for any $s \geq e^2\chi$ and any $\lambda \geq 1$, $\Pr \left(U  \geq s \right) \leq e^{-\chi} e^{-s}$.
\end{restatable}


Given Algorithm~\ref{alg:ocoea2} with constant $\chi>0$, we define the number of $1$-bits in each offspring in $t$ iteration as $X_t^{(i)}$ where $i\in [\lambda]$. Given the current number of $1$-bits $s \in [n]$ for the parent solution, we define the change of the number of $1$ in $X_t$ for each offspring by $\Delta X_t^{(i)} \sim V_1-V_2$ where $V_1 \sim Bin(n-s,\frac{\chi}{n})$ and $V_2 \sim Bin(s,\frac{\chi}{n})$. So $\Delta X_t^{(i)}$ has the same MGFs from Lemma~\ref{lem:MGF1} for each $i \in [n]$. From Lemma~\ref{lem:MaxPoisson}, for any $i \in [\lambda]$ and $s >0$, we have $\Pr \left(\Delta X_t^{(i)} > s \right) \leq e^{-\chi} \lambda ^{\chi} e^{- s\ln \ln \lambda }$.

\subsubsection{Phase 2.2}

Next, we show in a certain $c$-tube that the search point crosses the diagonal with high probability, resulting in the positive drift towards the optimum. First, we show the search point induced by Algorithm~\ref{alg:ocoea2} crosses the diagonal w.h.p.

\begin{restatable}{lem}{SecFivePhaseTwoTwolemOne}
\label{lem:crossD1}  Given problem size, offspring size and iteration $n, \lambda, t \in \mathbb{N}$, let $c:= \kappa \ln \lambda / \ln \ln \lambda$ for any constant $\kappa \in (0,1)$, we denote $E_t=$\{The algorithm crosses the diagonal as defined in Definition~\ref{def:cross} at iteration $t+1$\}. Assume any constants $\chi >0$ and $\varepsilon \in (0,1)$. If $D_t < c$, $n-X_t\geq \varepsilon n$ and $n-Y_t \geq \varepsilon n$, 
and if either of the two conditions holds $(1)$ $ t$ is even and $X_t < Y_t$; 
$(2)$ $ t$ is odd and $X_t \geq Y_t$, then $\Pr \left(E_t \right) \geq 1 -  2\left( 1 / \lambda^{1/2e^{\chi}} \right)$
\end{restatable}

Lemma~\ref{lem:crossD1} means that if the search point lies in a given tube of length $c$, then before reaching an $\varepsilon$-approximation, the search point keeps crossing the diagonal with high probability.
\par Next, we show that the search point deviates from some $c$-tube with a small probability. The idea is to show the algorithm is more likely to select the samples/offspring which lie inside the $c$-tube.  We first prove some lemma to proceed.

\begin{restatable}{lem}{SecFivePhaseTwoTwolemTwo}
\label{lem:crossD2} 
 Let problem size, offspring size and iteration $n, \lambda, t \in \mathbb{N}$, $n-X_t\geq \varepsilon n$ and $n-Y_t \geq \varepsilon n$ for any constant $\varepsilon \in (0,1)$ and constant mutation rate $\chi \in (0,1)$. For any $\kappa \in (\chi, (1+\chi)/2)$, let $c:=
 \kappa \ln \lambda / \ln \ln \lambda$,
 \begin{enumerate}
     \item[(1)] if $ t$ is even and $X_t < Y_t$, then $c$ satisfies the following conditions 
    \begin{enumerate}
        \item[(A)] \text{$\Pr\left(\max_{i \in [\lambda]} \Delta X_t^i \geq  D_t \mid D_t< c \right) \geq 1 -  2\left( 1 / \lambda^{1/2e^{\chi}} \right)$}
        \item[(B)] \text{$\Pr \left({K}/{M} \leq 2 (1+\delta)
         \left({1}/{\lambda}^{\kappa -\chi} \right) \right) \geq 1- e^{-\Omega(\lambda)}$ for any } \text{constant} $\delta \in (0,1)$
    \end{enumerate}
where $\Delta X_t^i:=X_t^i-X_t$, and $K=|\{i \mid \Delta X_t^i \geq D_t+c\}|$  and $M=|\{i \mid \Delta X_t^i \geq D_t\}|$.
        
     \item[(2)]  If $ t$ is odd and $X_t \geq  Y_t$, then $c$ satisfies 
    \begin{enumerate}
        \item[(C)] \text{$\Pr\left(\max_{i \in [\lambda]} \Delta Y_t^i \geq  D_t \mid D_t< c \right) \geq1 -  2\left( 1 / \lambda^{1/2e^{\chi}} \right)$}
        \item[(D)] \text{$\Pr \left({K'}/{M'} \leq 2 (1+\delta)
   \left({1}/{\lambda}^{\kappa -\chi} \right)  \right) \geq 1- e^{-\Omega(\lambda)}$ for any } \text{constant} $\delta \in (0,1)$
    \end{enumerate}
where $\Delta Y_t^i:=Y_t^i-Y_t$, and $K'=|\{i \mid \Delta Y_t^i \geq D_t+c\}|$  and $M'=|\{i \mid \Delta Y_t^i \geq D_t\}|$. 
 \end{enumerate}
\end{restatable}
Notice that in Lemma~\ref{lem:crossD2}, $\Delta X_t^i$ denotes the change of number of $1$-bits in $i$-th offspring of $x_t$, and $K=|\{i \mid \Delta X_t^i \geq D_t+c\}|$ means the number of  samples/offspring s.t. $\Delta X_t^i \geq D_t+c$ and $M=|\{i \mid \Delta X_t^i \geq D_t\}|$ means the number of samples/offspring s.t. $\Delta X_t^i \geq D_t$ and similar for $y_t$. $(A)$ and $(C)$ conditions in Lemma~\ref{lem:crossD2} means that for sufficiently large $\lambda$, the next search point produced by Algorithm~\ref{alg:ocoea2} can cross the diagonal with high probability. $(B)$ and $(D)$ conditions in Lemma~\ref{lem:crossD2} means that in those offspring which cross the diagonal, the portion of offspring which make a large jump to cross outside the $c$-tube is rare with overwhelmingly high probability.

\begin{restatable}{lem}{SecFivePhaseTwoTwolemThree}
(Escape from the $c$-tube with small prob.)
\label{lem:crossD3}
Assume the conditions of Lemma~\ref{lem:crossD2} hold.
Consider $(1,\lambda)$-CoEA on \Diagonal.
We define $H_t=2 n-X_t-Y_t$.
Let $T:=\inf \{t> 0 \mid H_t \leq \varepsilon n\}$ for any constant $\varepsilon \in (0,1)$.
For any $t\in  [1,T]$, any constants $\chi \in (0,1)$ and $\gamma \in (0,(1-\chi)/2)$, we have 
\begin{align*}
    \Pr(D_{t+1}>c \mid D_t <c) \leq 
    9 \left(\frac{1}{\lambda} \right)^{\gamma}
\end{align*}
where
$c$ is defined in Lemma~\ref{lem:crossD2}.
\end{restatable}

Lemma~\ref{lem:crossD3} means that if the search point lies in a given tube of length $c$, then before reaching $\varepsilon$-approximation, the search point stays in the given tube with probability $1-O\left(  1 / \lambda ^{\gamma} \right)$ for constant $\gamma>0$.


   We use Lemma~\ref{lem:crossD2} to show that firstly within $c$-tube, the search point can cross the diagonal with high probability. 
   Then, we consider two extreme cases when the search point lies in the tube. 
   Assume $Y_t>X_t$, the first one is that the search point is close to the upper boundary (with Hamming distance $1$), then we need to show it can still cross the diagonal but not cross too much and escape from the lower boundary with high probability. The second extreme case is if the search point stays very close to the diagonal (with Hamming distance $1$). 
   It is a challenge to show that the next search point will not escape from the lower boundary since we have already shown within the tube, that there is a high probability that the offspring jump at least Hamming distance $D_t$ to cross the diagonal. 
   So, it is the case that a few offspring may escape from the lower boundary. 
   The key is to observe that more offspring cross the diagonal lie inside the tube compared with those outside the tube. 
    All the offspring in Algorithm~\ref{alg:ocoea2} which cross the diagonal have fitness $1$. 
   Then Algorithm~\ref{alg:ocoea2} selects the next search point uniformly at random from these offspring crossing the diagonal and with an overwhelming higher probability to select those inside the tube. 

\begin{restatable}{lem}{SecFivePhaseTwoTwolemFour}
\label{lem:crossD4}Let problem size, offspring size and iteration $n, \lambda, t \in \mathbb{N}$, $n-X_t\geq \varepsilon n$ and $n-Y_t \geq \varepsilon n$ for any constant $\varepsilon \in (0,1)$. For any $\chi \in (0,1)$, let  $D_t \in [0,c]$ where $c=(\frac{1+3\chi}{4}) \ln \lambda/{\ln \ln \lambda}$.  If $t$ is even and $Y_t>X_t$, then $X_{t+1}-X_t\geq 1$ with probability at least $1-O({1}/{\lambda ^{(1-\chi)/4}})$. If $t$ is odd and $Y_t\leq X_t$, then $Y_{t+1}-Y_t\geq 1$ with probability at least $1-O({1}/{\lambda ^{(1-\chi)/4}})$.
Moreover, let $H_t=2n-X_t-Y_t$, then $\Et{H_t-H_{t+1}} \geq 1/2$.
\end{restatable}

From Lemma~\ref{lem:crossD4}, before reaching $\varepsilon$-approximation, if the search point stays within the $c$-tube, there exists positive constant drift towards the optimum when considering $H_t$ as the Hamming distance to $(n,n)$. Since we show the existence of positive drift, next we come to the main theorem of this paper.




\subsubsection{Phase 2.3}

Now, we wrap up everything and use a restart argument to compute the overall runtime by using Lemma~\ref{lem:phase1} and Lemma~\ref{lem:crossD4}. 
Moreover, we will restart the algorithm every $2T_c$ generation with $x_0=y_0=0^n$.
Given the problem size $n \in \mathbb{N}$, we have the main result:

\begin{restatable}{theorem}{SecFivePhaseTwoThreeMainOne}
\label{thm:Main} 
Consider the $(1,\lambda)$-CoEA with constant mutation rate
$\chi\in(0,1)$, and offspring size $\lambda \geq
(264(2-\delta)n)^{4/(1-\chi)}$ for any constant $\delta \in (0,1)$. Assume that the algorithm is
initialised at $x=y=0^n$, and restarted at $x=y=0^n$ every $2T_c$ generations, with $T_c:=4(2-\delta)n$. Then
the expected time to find a $\delta$-approximation to the \maximin optimum of \Diagonal is at most $24 (2-\delta)\lambda n$.
\end{restatable}

Theorem~\ref{thm:Main} shows that a large offspring population size helps the CoEA to cross the diagonal consistently with high probability, which is thus favoured during selection and finally leads to the positive drift towards the optimum. 
Lemma~\ref{lem:crossD2} and Lemma~\ref{lem:crossD3} show the necessity of large offspring size. Otherwise, a small number of offspring will let CoEA fall to the flat fitness landscape apart from the tube along the diagonal and then get lost in the random walk around the search space.
Above all, with proper design of coevolution and large offspring population size, $(1,\lambda)$-CoEA can find an $\varepsilon$-approximation to the optimum of \Diagonal efficiently.

\section{Experiments}
To complement our asymptotic results with data for concrete problem sizes, we conduct the following experiments. 
\subsection{Settings}
We conduct experiments with the $(1,\lambda)$-CoEA on \Diagonal problem with $n=\lambda=1000$ and the initialisation of the algorithm is set up as uniformly at random. 
We also set different mutation rates in the range of $0.2$ to $2.2$ in increments of $0.2$ and conduct $100$ independent runs for each configuration to explore the suitable range of mutation rate for the \Diagonal problem. The budget for each run is set to be $10^9$ function evaluations.
\begin{figure}[htpb]
        \centering
         \begin{subfigure} {0.495\textwidth}
             \centering
             \includegraphics[width=\textwidth]{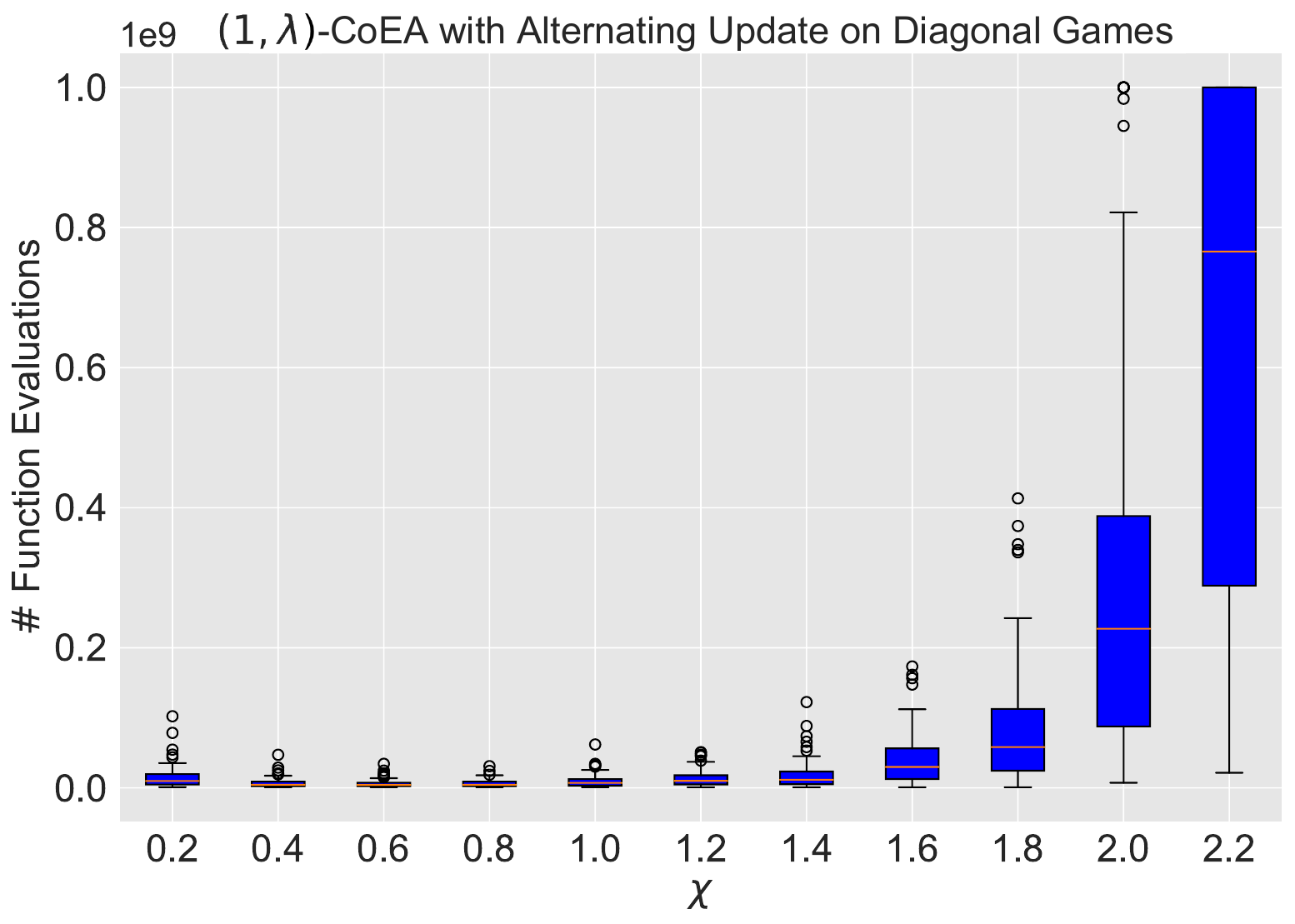}
             \caption{Runtime from $\chi=0.2$ to $2.2$.}
             \label{fig:Diagonal1}
         \end{subfigure}
         \hfill
         \begin{subfigure} {0.495\textwidth}
             \centering
             \includegraphics[width=\textwidth]{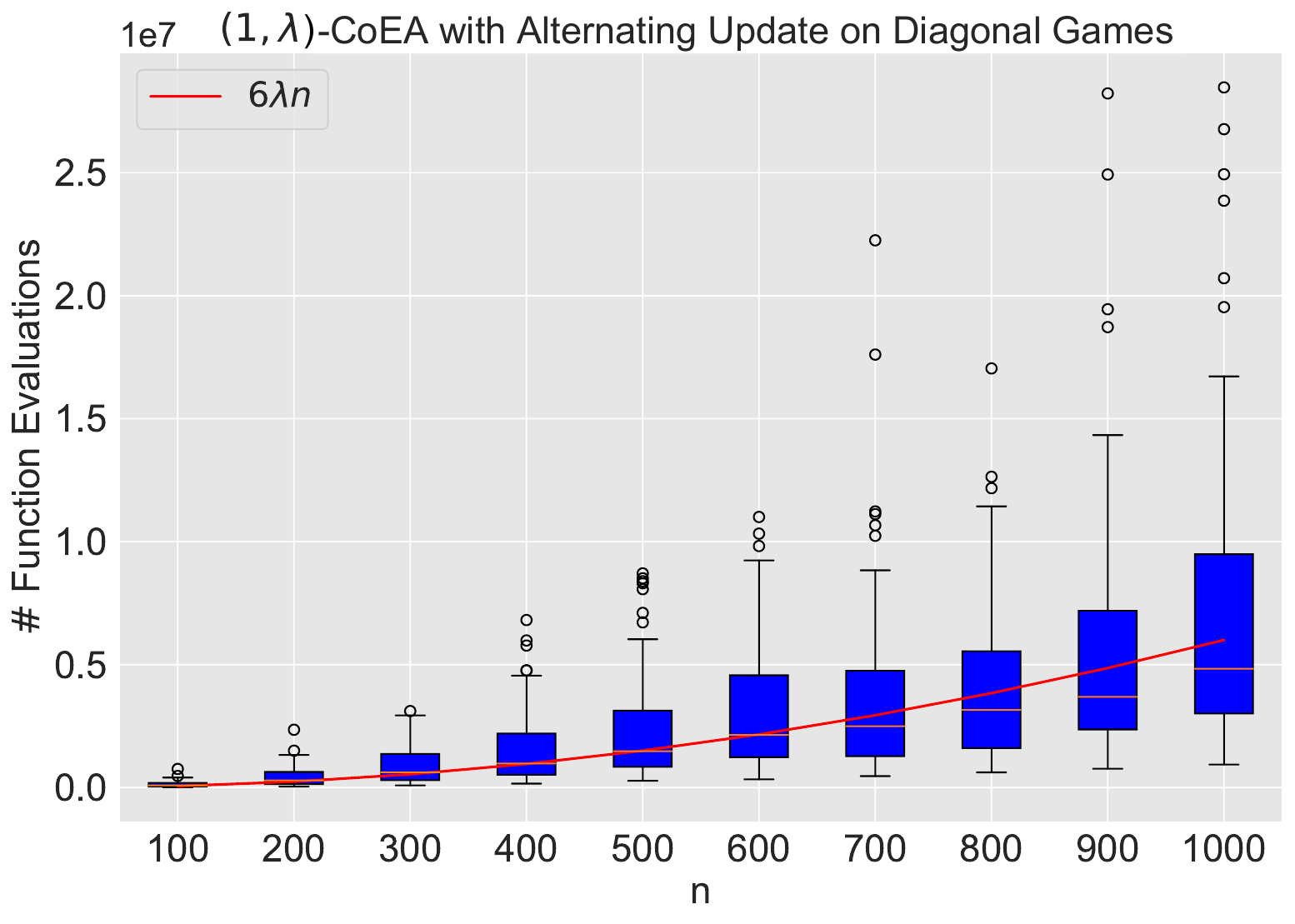}
             \caption{Runtime from $n=100$ to $n=1000$.}
             \label{fig:Diagonal2}
         \end{subfigure}
        \caption[Examples]{
        Runtime of $(1,\lambda)$-CoEA on \Diagonal.
        Fig.~\ref{fig:Diagonal1} (\textbf{left}): Runtime against different mutation rates under $n=\lambda=1000$.
        Fig.~\ref{fig:Diagonal2} (\textbf{right}): Runtime against specific $\chi =0.6$. The red curve is $f(n,\lambda)=6 \lambda n 
        $ where in this case we set $n=\lambda$.
        }
\end{figure}
\par Then, as Fig~\ref{fig:Diagonal1} shows, $\chi=0.6$, the expected performance of $(1,\lambda)$-CoEA is the best. Then, we fix such a mutation rate, and run the experiments in the range of $n=100$ to $n=1000$ in increments of $100$. We conduct $100$ independent runs for each configuration. The budget for each run is $10^9$ function evaluations.


\subsection{Results}
Fig~\ref{fig:Diagonal1} displays the runtime of $(1,\lambda)$-CoEA on \Diagonal for different mutation rate from $0.2$ to $2.2$. This data confirms that for the suitable low mutation rates and sufficiently large offspring size, Algorithm~\ref{alg:ocoea2} finds the optimum efficiently. The higher mutation rate eventually leads to the inefficiency of the algorithm. As we observe in Fig~\ref{fig:Diagonal2}, under suitable mutation rate $\chi=0.6$, the empirical average of the runtime is bounded above by $O(\lambda n )$. Notice that in our theoretical analysis, we need to require $\lambda=\Omega(n^{4/(1-\chi)})$, while it seems that 
$\lambda=O(n)$ in the experiments is already sufficient to guarantee a polynomial runtime for $(1,\lambda)$-CoEA on \Diagonal. This suggests that the current bound may not be tight and our theoretical analysis still has room to improve.

\section{Discussion and Conclusion}
CoEAs exhibit complex dynamics on \maximin optimisation. 
To the best of our knowledge,  this paper is the first runtime analysis of CoEAs on binary test-based adversarial optimisation problems. As a starting point, we propose a binary test-based problem called \Diagonal.
We showed that for \Diagonal, 
$(1,\lambda)$-EA get trapped in binary fitness landscape.
Thus, traditional $(1,\lambda)$-EA failed to find any approximation to optimum in polynomial runtime due to negative drift induced by flat fitness landscape. 
However, for (1,$\lambda$)-CoEA with the alternating update method, if the offspring population is sufficiently large $\lambda = \Omega(n^{4/(1-\chi)})$  with a reasonable constant mutation rate $\chi\in (0,1)$, it can find an approximation to optimum efficiently in expected runtime $O\left(\lambda  n \right)$. 
We want to highlight the necessity of coevolution and large offspring size in solving these binary problems in which the fitness landscape is very flat and hard to search from these analyses.  
\par On the technical side, this paper shows that mathematical runtime analyses are also feasible for the $(1,\lambda)$-CoEA. 
We are optimistic that our tools will widen the toolbox for future analyses of competitive CoEAs. 
On the practical side, it brings insight for practitioners that traditional EAs may not be well suited for \Diagonal-like problems. 
We suggest using CoEAs with large samples and relatively low mutation rates, which can help to search Nash Equilibria on binary problems with similar hard and flat fitness landscapes more efficiently.
\par For future work, it is interesting to provide a more precise upper bound for the runtime of $(1,\lambda)$-CoEAs on \Diagonal and a more general analysis by relaxing the deterministic initialisation since the empirical results suggest our theoretical bound might not be tight enough. 
Using more advanced theoretical tools like \cite{lehre2024concentration}, we can derive a better tail bound of the current runtime for $(1,\lambda)$-CoEA.
It is also worth exploring whether there are more efficient ways to capture the interaction between two populations well, for example, any combination of coevolution and self-adaptation or multi-objective optimisation~\cite{qin2022self}.
Furthermore, it is exciting to explore the behaviour of CoEAs on a more general class of payoff functions like generalised boundary test-based problems.

\section*{Acknowledgments}
This work was supported by a Turing AI Fellowship (EPSRC grant ref 
EP/V025562/1).
The computations were performed using the University of Birmingham’s BlueBEAR high performance computing (HPC) service.

\section*{Disclosure of Interest}
No, I declare no competing interests as defined by Springer Nature, or other interests that might be perceived to influence results and/or discussion reported in this manuscript.

\bibliographystyle{splncs04}


\newpage


\section*{Proof Idea for Theorem 3}

\begin{figure}[!ht]
        \centering        \includegraphics[width=0.7\linewidth]{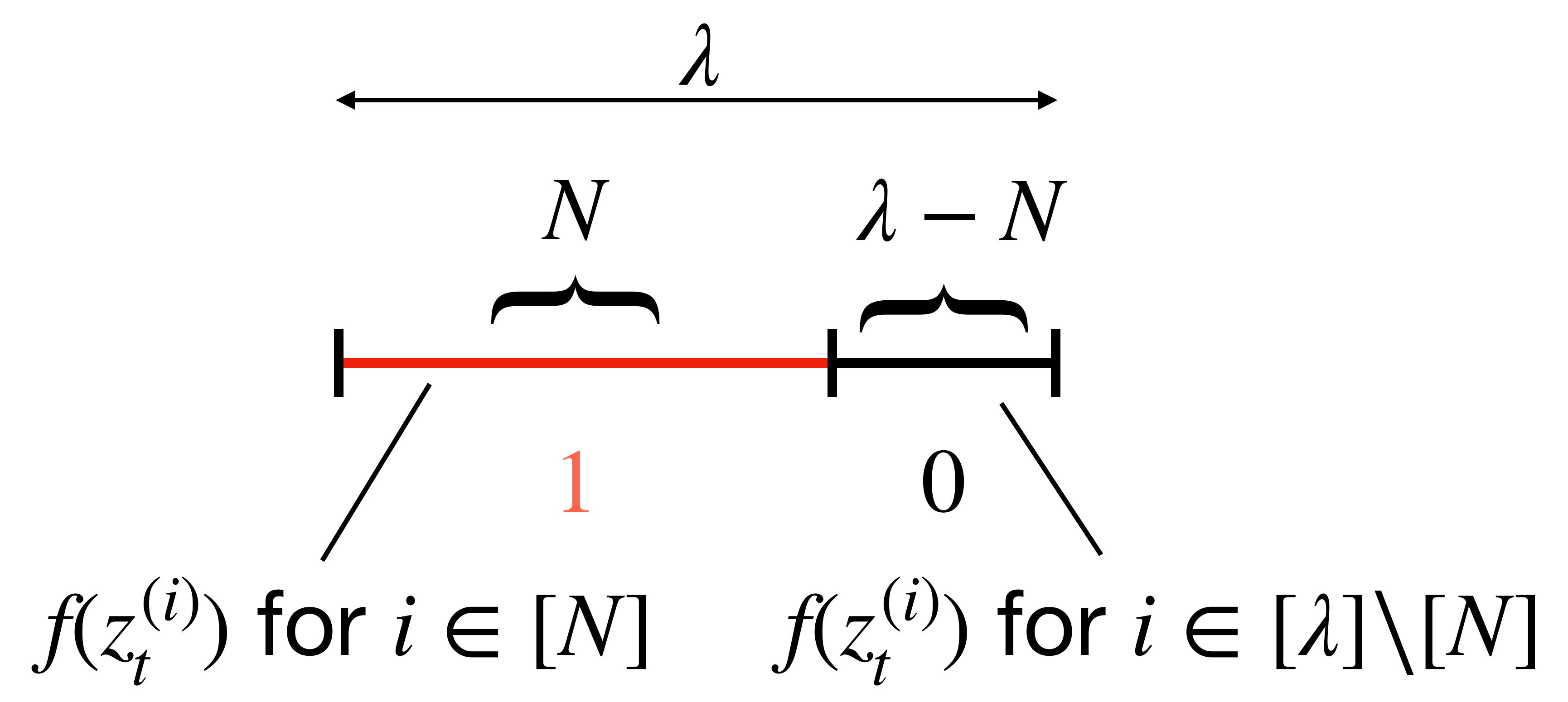}
        \caption[Examples]{Sketch of proof idea for Theorem~\ref{thm:Main0}.
        The offspring are sorted according to fitness. 
        We denote the number of offspring with fitness $1$ by $N$.
        }
        \label{fig:Thmthree}
\end{figure}

\section*{Omitted Proofs}

\SecFourMainZero*

\begin{proof}[Proof of Theorem~\ref{thm:Main0}]
As mentioned previously, $(1^n,1^n)$ is the \maximin optimum or \nash Equilibrium of $f$. 
Note that $(1,\lambda)$-EA initialises the search point uniformly at random. 

Let us denote the number of $1$-bits of search point at iteration $t$ by $Z_t$. Let $M_t=2n-Z_t$.
Next, to show the runtime's lower bound, we define the first hitting time $T_{\varepsilon}:=\inf \{t>0 \mid M_t\leq \varepsilon n\}$ for any $\varepsilon \in [0,1)$.
Note that $Z_0$ are subject to Binomial distribution $\text{Bin}(2n,1/2)$.
We compute the probability of $M_0\leq (1-\eta) n$ for any $\eta \in (0,1)$.

\begin{align}
    \Pr \left(M_0\leq  (1-\eta) n \right)
    = \Pr \left(2n-Z_0 \leq  (1-\eta) n \right) 
    &=  \Pr \left( Z_0 \geq (1+\eta) n \right) \nonumber \\
    \intertext{By using Chernoff's bound on $Z_0 \sim \text{Bin}(2n,1/2)$, we have}
    &\leq \exp \left(-\frac{n \eta^2}{2+\eta} \right)  = e^{-\Omega(n)}.
    \label{eq:initThmNeg}
\end{align}
This means that the initial search point $Z_0$ satisfies $M_0>(1-\eta)n$ with probability $1-e^{-\Omega(n)}$ for any $\eta \in (0,1)$ or with high probability for sufficiently large $n$. 
We satisfy one of the assumptions of Theorem~\ref{thm:NegativeDrift} with high probability.

Next, we denote $x_t, y_t$ by the binary bitstring for the leftmost part and rightmost part of the original bitstring and 
$X_t, Y_t$ by the numbers of $1$-bits of $x, y$ respectively at iteration $t$. 
Suppose the numbers of  $1$-bits of $\lambda$ offspring from $x, y$ respectively at iteration $t$ are denoted by 
\begin{align*}
    (X_t^{(1)},Y_t^{(1)}), (X_t^{(2)},Y_t^{(2)}), \cdots ,(X_t^{(\lambda)},Y_t^{(\lambda)}).
\end{align*}
Based on the bitwise mutation operator, we have, for $i\in [\lambda]$,
\begin{equation}
    X_t^{(i)} \sim \text{Bin}(n-X_t,\frac{\chi}{n})-\text{Bin}(X_t,\frac{\chi}{n}) \text{ and } Y_t^{(i)} \sim \text{Bin}(n-Y_t,\frac{\chi}{n})-\text{Bin}(Y_t,\frac{\chi}{n}).
    \label{eq:mutant}
\end{equation}

If the search point $z_t=(x_t,y_t)$ satisfies $X_t<Y_t$, then $f(z_t)=0$. 
 $(1,\lambda)$-EA continues the random walk on the search space of fitness $0$ until it finds the point of fitness $1$. 
Next, we show the following claim:
\begin{quote}
     Under the worst case scenario, this random walk firstly leads to $(X_t,Y_t)=(\frac{n}{2}, \frac{n}{2})$ in polynomial runtime rather than $(X_t,Y_t)=(n,n)$ within exponentially large iterations.
\end{quote} 
Suppose the $(1,\lambda)$-EA is unlucky and samples all the offspring of fitness $0$, then the algorithm will choose one of the offspring uniformly at random. 
In this case, the search point will be only driven by the mutation operator. 
So we compute the drift
\begin{align*}
    \Et{X_t-X_{t+1}}=(n-2X_t)\frac{\chi}{n} \text{ and }  
    \Et{Y_t-Y_{t+1}}=(n-2Y_t)\frac{\chi}{n}.
\end{align*}
Firstly, we consider the hitting time to $(X_t,Y_t)=(n/2,n/2)$.
Without loss of generality, if $X_t\leq \frac{n}{2}$ (or $Y_t\leq \frac{n}{2}$), then the additive drift theorem implies that the algorithm will reach $(n/2,n/2)$ with the potential function $M_t=\frac{n}{2}-X_t$ (or $M_t=\frac{n}{2}-Y_t$).
If $X_t \geq \frac{n}{2}$ (or $Y_t\geq \frac{n}{2}$), by taking $M_t:=X_t-\frac{n}{2}$ (or $M_t:=Y_t-\frac{n}{2}$) as the potential function, we can obtain the same result.
Now, if we consider the hitting time to $(X_t,Y_t)=(n,n)$ with the uniform initialisation of the search point (shown in Eq.~\ref{eq:initThmNeg}), then we can only obtain the negative drift before reaching the optimum.
Moreover, since it is a random walk, and each of the offspring is of the same fitness $0$, then we satisfy condition (2) in Theorem~\ref{thm:NegativeDrift}. Using Theorem~\ref{thm:NegativeDrift}, we can conclude that this random walk hits $(X_t,Y_t)=(n,n)$ with an exponentially large runtime.

Now, we assume that before reaching the optimum $(1^n,1^n)$, $X_t\geq Y_t$ and $
\varepsilon n<M_t <n/4$ and $Y_t\geq n/2$.
We define the number of offspring such that $\Delta_t^{(i)}:=X_t^{(i)}-Y_t^{(i)}\geq 0$ by $N$:
\begin{align*}
    N:=|\{i \in [\lambda]\}\mid  \Delta_t^{(i)}\geq 0\}|.
\end{align*}
We can see that for each $i\in [\lambda]$,
\begin{align*}
    \Pr\left( \Delta_t^{(i)}\geq 0 \right) \geq (1-\frac{\chi}{n})^{2n} \geq \frac{1}{2e^{2\chi}}:=q_{N^*}.
\end{align*}
where we define $N^* \sim \text{Bin}(\lambda,q_{N^*})$. Then, we can see $N$ stochastic dominates $N^*$ (i.e. $N^* \succcurlyeq N$) since $N \sim \text{Bin}(\lambda, \Pr\left( \Delta_t^{(i)}\geq 0 \right))$ where $ \Pr\left( \Delta_t^{(i)}\geq 0 \right) \geq q_{N^*}$. 
By using stochastic dominance, we obtain: for any $\delta \in (0,1)$,
\begin{align*}
           \Pr \left(N\leq (1-\delta) \lambda q_{N^*} \right) 
    &\leq  \Pr \left(N^*\leq (1-\delta) \lambda q_{N^*} \right) \\
    &=     \Pr \left(N^*\leq (1-\delta)  \E{N^*} \right) \\ \intertext{Using Chernoff's bound gives}
    &\leq \exp \left(- \E{N^*} \delta /2 \right) = e^{-\Omega(\lambda)}.
\end{align*}
Note that $\Delta_t^{(i)}\geq 0$ means that we have $f(z_t^{(i)})=1$ and the above concentration means that when sampling $\lambda$ offspring, 
there are at least $(1-\delta)\lambda /2e^{2\chi}$ offspring of fitness $1$ with probability $1-e^{-\Omega (\lambda)}$.
Then, for any $j \in \mathbb{N}_0$, using the law of total probability gives 
\begin{align*}
    \Prt \left(|M_t-M_{t+1}|\geq j \right) 
    &=  \Prt \left(|M_t-M_{t+1}|\geq j \mid N\geq \frac{\lambda}{4e^{2\chi}}\right)\cdot \Pr( N\geq \frac{\lambda}{4e^{2\chi}})\\
    &\quad +  \Prt \left(|M_t-M_{t+1}|\geq j \mid N\leq \frac{\lambda}{4e^{2\chi}}\right) \cdot \Pr ( N\leq \frac{\lambda}{4e^{2\chi}}) \\ \intertext{Using $ \Pr( N\geq \frac{\lambda}{4e^{2\chi}})\leq 1$ and $\Pr ( N\leq \frac{\lambda}{4e^{2\chi}}) \leq e^{-\Omega (\lambda)}$ gives}
    &\leq  \Prt \left(|M_t-M_{t+1}|\geq j \mid N\geq \frac{\lambda}{4e^{2\chi}}\right)+ e^{-\Omega (\lambda)} \\ 
    \intertext{Recall that $X_t^{(i)},Y_t^{(i)}$ denote the number of $1$ bits of $i$-th offspring for $x,y$ individual respectively (for $i \in [\lambda]$).
    Let us denote the difference between Hamming distance to the optimum for each $i$-th offspring (i.e $2n-X_t-Y_t-(2n-X_t^{(i)}-Y_t^{(i)})=(X_t^{(i)}-X_t)+(Y_t^{(i)}-Y_t)$) by $Z_t^{(i)}:=(X_t^{(i)}+Y_t^{(i)})-X_t-Y_t$
    and the event $\{|M_t-M_{t+1}|\geq j\}$ by $E_j$. 
    We also denote the event $\{\text{$k$-th offspring is selected}\}$ by $F_k$.
    Using the law of total probability for conditional probability gives}
    &= \sum_{k=1}^{\lambda} 
    \Pr \left(E_j \mid N\geq \frac{\lambda}{4e^{2\chi}}, F_k \right)
    \Pr \left(F_k \mid N\geq \frac{\lambda}{4e^{2\chi}} \right) + e^{-\Omega (\lambda)}   \\ 
    \intertext{When conditioning on $F_k$, the event $E_j$ becomes $|Z_t^{(k)}|\geq j$. Thus, we have}
    &=  \sum_{k=1}^{\lambda} 
    \Pr \left(|Z_t^{(k)}|\geq j \mid  N\geq \frac{\lambda}{4e^{2\chi}}, F_k\right)\Pr \left(F_k \mid N\geq \frac{\lambda}{4e^{2\chi}} \right) + e^{-\Omega (\lambda)}  \\ \intertext{We can upper bound probability of $|Z_t^{(k)}|\geq j$ by only considering flipping $j$ bits among $2n$-bitstring. Thus,
    }
    &\leq \sum_{k=1}^{\lambda} \binom{n}{j} (\frac{\chi}{n})^j \Pr\left(F_k \mid N\geq \frac{\lambda}{4e^{2\chi}} \right) + e^{-\Omega (\lambda)}
    \\ \intertext{Using $\binom{n}{j} \leq \frac{n^j}{j!}, j!\geq 2^j $ and  $\chi \in (0,1]$ gives}
    &\leq \sum_{k=1}^{\lambda} \frac{1}{2^j} \Pr\left(F_k \mid N\geq \frac{\lambda}{4e^{2\chi}} \right) + e^{-\Omega (\lambda)}
    \intertext{Recall that we denote $N$ to the number of offspring of fitness $1$ (i.e.$\Delta_t^{(i)} \geq 0$) and $F_k$ denotes the event that $k$-th offspring will be selected. 
    As mentioned, when the algorithm comes across a tie, we select the offspring uniformly at random among those offspring of fitness $1$. 
    Thus, we can get}
    &\leq \lambda \cdot \frac{1}{2^j}\cdot \frac{1}{N} +e^{-\Omega(\lambda)} \\ \intertext{Also, note that we condition on the event that $N \geq \lambda / 4e^{2\chi}$.}
    &\leq  \frac{4e^{2\chi}}{2^j}+ e^{-\Omega (\lambda)} \leq \frac{8e^{2\chi}}{2^j}.
\end{align*}
Now, we satisfy condition (2) in Theorem~\ref{thm:NegativeDrift} with $r=8e^{2\chi}=o(n/\log n)$ and $\eta=1$.

Finally, we compute the negative drift for $M_t$. Recall the assumptions: 
before reaching the optimum $(1^n,1^n)$, $X_t\geq Y_t$ and $
\varepsilon n<M_t <n/4$ and $Y_t\geq n/2$. 
It is possible for $(1,\lambda)$-EA to produce a search point which is allowed to move to $4$ directions in $|x|-|y|$ plane (see Figure~\ref{fig:DA1commalambda1}) only if the search point is of fitness $1$ (or below the diagonal in Figure~\ref{fig:DA1commalambda1}).

Next, we compute the drift.
Then we can compute the following:
\begin{align*}
    \Et{M_t-M_{t+1}}
    &=\Et{X_t-X_{t+1}} + \Et{Y_t-Y_{t+1}} \\ \intertext{
    Recall that $X_t^{(i)},Y_t^{(i)}$ denote the number of $1$ bits of $i$-th offspring for $x,y$ individual respectively (for $i \in [\lambda]$).
    Recall that $Z_t^{(i)}:=(X_t^{(i)}+Y_t^{(i)})-X_t-Y_t$
    and that the event $\{\text{$k$-th offspring is selected}\}$ by $F_k$. Using law of total expectation gives
    }
    &= \sum_{k=1}^{\lambda}\E{M_t-M_{t+1} \mid F_k} \Pr \left(F_k \right) \\ 
    &= \sum_{k=1}^{\lambda}\E{Z_t^{(k)}}\Pr \left(F_k \right)  \\ \intertext{Recall the bitwise mutation operator in Eq.~\ref{eq:mutant}. Each $Z_t^{(k)}$ is independently and identically distributed.}
    &=  \E{Z_t^{(k)}} \sum_{k=1}^{\lambda}\Pr \left(F_k \right)
    \intertext{Note that $\sum_{k=1}^{\lambda}\Pr \left(F_k \right)=1$.
    We can prove it by dividing it into two cases. 
    Firstly, if $N\geq 1$, then
    the probability of each offspring of fitness $1$ being selected is $1/N$ and the probability of each offspring of fitness $0$ being selected is $0$ if $N\geq 1$.
    Thus, if we sort all the offspring of fitness $1$ in the first $N$ offspring and the offspring of fitness $0$ in the rest $\lambda-N$ places (See Fig.~\ref{fig:Thmthree}), then
    $\sum_{k=1}^{\lambda}\Pr \left(F_k \right)=\sum_{k=1}^N 1/N=1$.
    Secondly, if $N=0$, then $\sum_{k=1}^{\lambda}\Pr \left(F_k \right)=\sum_{k=1}^{\lambda}\frac{1}{\lambda}=1$.}
    &=  \E{Z_t^{(k)}}
    \\ \intertext{To simplify the calculation, we underestimate the negative drift along $x$-axis by neglecting the backwards negative drift and only considering the forward positive drift towards $x=1^n$.
    Moreover, we overestimate the positive drift along the $y$-axis by taking all the possible samples (even the search point above the diagonal and thus of fitness $0$) into account. }
    &\leq   (n-X_t)\frac{\chi}{n} + (n-2Y_t) \frac{\chi}{n} \\ 
    &=M_t \frac{\chi}{n} - Y_t\frac{\chi}{n} \leq \frac{1}{4} - \frac{1}{2} = \frac{-1}{4}<0.
\end{align*}
Now, we satisfy all the conditions of Theorem~\ref{thm:NegativeDrift}.
Then, we apply Theorem~\ref{thm:NegativeDrift} on $M_t$ to derive the following: there exists a constant $c>0$ such that
\begin{align*}
      \Pr\left( T_{\varepsilon} \leq e^{cn} \right) \leq e^{-cn}.
\end{align*}
\qed
\end{proof}

\SecFivelemOne*

\begin{proof}[Proof of Lemma~\ref{lem:AvgCharacteristic1}]
    \begin{itemize}
    \item[(1)] If $|x^{(1)}|\geq |x^{(2)}|$, then we rewrite the fitness in terms of sum indicator functions:
    \begin{align*}
        f(x^{(1)}):= g(x^{(1)},y)= \mathds{1}_{\{x^{(1)}\geq y \}}
    \end{align*}
    Notice that if $|x^{(2)}|\geq |y|$, then $|x^{(1)}|\geq |x^{(2)}| \geq |y|$.  Then we have the event inclusion:
    \begin{align*}
        \{|x^{(2)}|\geq |y| \} &\subseteq \{|x^{(1)}| \geq |y|\} \\ \intertext{This implies that}
        \mathds{1}_{\{|x^{(2)}|\geq |y| \}} & \leq \mathds{1}_{\{|x^{(1)}| \geq |y|\}} \\ \intertext{Then, we have}
         f(x^{(2)})= \mathds{1}_{\{|x^{(2)}|\geq |y| \}} & \leq   \mathds{1}_{\{x^{(1)}\geq y \}}=f(x^{(1)}).
    \end{align*}
    
    \item[(2)] If $|y^{(1)}|\geq |y^{(2)}|$, then we rewrite the fitness in terms of sum indicator functions:
    \begin{align*}
        h(y^{(1)}):=g(x,y^{(1)})=  \mathds{1}_{\{x \geq y^{(1)} \}}
    \end{align*}
    Notice that if $|x|\geq |y^{(1)}|$, then $|x|\geq |y^{(1)}| \geq |y^{(2)}|$.  Then we have the event inclusion:
    \begin{align*}
        \{|x|\geq |y^{(1)}| \} &\subseteq \{|x| \geq |y^{(2)}|\} \\ \intertext{This implies that}
        \mathds{1}_{\{|x|\geq |y^{(1)}| \}} & \leq \mathds{1}_{\{|x| \geq |y^{(2)}|\}} \\ \intertext{Then, we have}
        h(y^{(1)}) = \mathds{1}_{\{|x|\geq |y^{(1)}| \}} & \geq    \mathds{1}_{\{|x| \geq |y^{(2)}|\}} = h(y^{(2)}).
    \end{align*}
\end{itemize}
\qed
\end{proof}

\SecFivePhaseOneMainOne*
\begin{proof}[Proof of Lemma~\ref{lem:phase1}]
In odd iterations $t=1,3, \dots, $ Algorithm~\ref{alg:ocoea2} updates $y$ and in even iterations $t=2,4, \dots, $ Algorithm~\ref{alg:ocoea2} updates $x$. 
\par If we have $Y_{2t}>X_{2t}$ in iteration $2t$, then the next search point does not cross the diagonal if $Y_{2t+1}>X_{2t+1}$ or $X_{2t+2} \geq Y_{2t+2}$ in one cycle. And if $Y_{2t+1}+c \leq X_{2t+1}$ or $X_{2t+2}+c \leq Y_{2t+2}$, then the search point lies outside the $c$-tube during the consecutive steps. Thus, a successful cycle in Definition~\ref{def:phase1} breaks.  
Note that a successful cycle in Definition~\ref{def:phase1} breaks is equivalent to the event that either the search point escapes from the $c$-tube or does not cross the diagonal during the cycle. Then, we use the union bound to give an upper bound of the probability that a cycle fails to be successful. Before that,  we provide a more precise union bound as follows. Given events $A,B,E$, 
\begin{align}
    \Pr \left(A \cup B \mid E \right) 
    &=\Pr\left(A \mid E \right) + \Pr \left(B \setminus A \mid E \right) \nonumber \\
    &\leq \Pr\left(A \mid E \right) + \Pr \left(B \mid E, A^c \right) \Pr \left(A^c \mid E \right)  \nonumber \\
    &\leq \Pr\left(A \mid E \right) + \Pr \left(B \mid E, A^c \right)  \label{eq:unionbound}
\end{align}
Next, we denote the following events.
\begin{align*}
    E_0:&=\{Y_{2t}>X_{2t} \cap D_{2t}<c\} \\
    E_1:&=\{X_{2t+1}\geq Y_{2t+1} \} \\
    E_2:&=\{Y_{2t+2} > X_{2t+2} \} \\
    F_1:&=\{D_{2t+1}\geq c\} \\
    F_2:&=\{D_{2t+2}\geq c\} 
\end{align*}
Note that we have a successful cycle in iteration $2t$ if and only if $E_1\cap E_2\cap F_1^c \cap F_2^c$ holds.
\begin{align}
        &\Pr \left(\text{A cycle fails to be successful at iteration $2t$} \right) \nonumber \\
    =   &  \Pr \left(E_1^c \cup E_2^c\cup F_1 \cup F_2 \mid E_0 \right) \nonumber \\ \intertext{Using the union bound gives}
    \leq & \Pr \left(E_1^c \cup E_2^c \mid E_0 \right)+ \Pr \left(F_1 \cup F_2 \mid E_0 \right)\nonumber \\ \intertext{Using Eq~(\ref{eq:unionbound}) gives}
    \leq& \Pr \left(E_1^c \mid E_0 \right) + \Pr \left(E_2^c \mid E_0, E_1 \right)  \nonumber\\
    &\quad + \Pr \left(F_1 \mid E_0 \right) + \Pr \left(F_2 \mid E_0, F_1^c \right)  \nonumber\\ \intertext{Using the conditions $(1)-(3)$ gives}
    \leq & 2p_c + 2p_e \label{eq:cycle}
\end{align}
Next, we consider $\tau$ consecutive cycles. We define the failure event $F_t$ as the algorithm has an unsuccessful cycle in iteration $t$. By using the union bound and Eq~(\ref{eq:cycle}), we have
\begin{align*}
    &\Pr \left(\exists t \in [\tau] \text{ s.t. }F_t \text{ occurs during $\tau$ consecutive cycles}  \right) \\
    \leq& 2\tau (p_c+p_e)
\end{align*}
So, we can derive that
\begin{align*}
    &\Pr \left(\text{The algorithm has $\tau$ consecutive successful cycles} \right) \\
    \geq& 1 - 2\tau (p_c+p_e).
\end{align*}
A similar analysis holds for the case that $(X_{2t+1},Y_{2t+1})$ in iteration $2t+1$.

\qed 
\end{proof}

\SecFivePhaseTwoOnelemOne*
\begin{proof}[Proof of Lemma~\ref{lem:binomial}]
    Consider the Binomial Theorem:
    \begin{align}
        \left(p+ (1-p) \right)^n &= \sum_{k=0}^n \binom{n}{k}p(1-p)^{n-k} \nonumber \\ 
                                 &= \sum_{k=0}^{\floor{n/2}} \binom{n}{2k}p^{2k}(1-p)^{n-2k}   \nonumber \\
                                 &\quad + \sum_{k=0}^{\floor{n/2}-1} \binom{n}{2k+1}p^{2k+1}(1-p)^{n-(2k+1)} \label{eq:binomial1}
    \end{align}
    And 
    \begin{align}
        \left(p- (1-p) \right)^n &= \sum_{k=0}^n \binom{n}{k}p \left(-(1-p)\right)^{n-k} \nonumber \\ 
                                 &= \sum_{k=0}^{\floor{n/2}} \binom{n}{2k}p^{2k}(1-p)^{n-2k}   \nonumber \\
                                 &\quad - \sum_{k=0}^{\floor{n/2}-1} \binom{n}{2k+1}p^{2k+1}(1-p)^{n-(2k+1)} \label{eq:binomial2}
    \end{align}
Now, we add Equations \ref{eq:binomial1} and \ref{eq:binomial2} to obtain:
\begin{align*}
     \Pr\left(\text{$Z$ is even} \right) &= \sum_{k=0}^{\floor{n/2}} \binom{n}{2k}p^{2k}(1-p)^{n-2k}\\
                                         &= \frac{1}{2} \left(1^n + (2p-1)^n \right).
\end{align*}
\qed
\end{proof}

\SecFivePhaseTwoOnelemTwo*
\begin{proof}[Proof of Lemma~\ref{lem:crossD0}]
If $\chi=O(1)$, then
\begin{align*}
    \Pr\left(\text{No offspring of $x_t$ is identical to $x_t$}\right) &= \left(1-(1-\frac{\chi}{n})^n \right)^{\lambda} \\ \intertext{Here, we use $(1-\frac{\chi}{n})^n \geq e^{-\chi}(1-\frac{\chi^2}{n})$ and taking $n$ large enough such that $(1-\frac{\chi^2}{n}) \geq \frac{1}{2}$ gives }
                        & \leq  (1-\frac{e^{-\chi}}{2})^{\lambda} \\ \intertext{Note that $1-\frac{e^{-\chi}}{2}$ is a constant in $(0,1)$ so that we obtain the exponentially small tail in $\lambda$. }
                        & \leq e^{-\Omega(\lambda)}.
\end{align*}
Then there exists $k, \ell>0$ s.t. $x_t^{k}=x_t$ and $y_t^{\ell}=y_t$ w.h.p. respectively.    
\par Moreover, note that $\Pr(X_t^{(1)}>X_t^{(\lambda)})=1-\Pr(X_t^{(1)}\leq X_t^{(\lambda)})= 1- \Pr(X_t^{(1)}=X_t^{(\lambda)})$. Then we can bound the following first (denote the number of 1-bits of original $x_t$ by $X_t$):
\begin{align*}
    \Pr\left(X_t^{(i)}=X_t\right) &= \sum_{k=0}^{\min \{n-X_t, X_t\}}\binom{n-X_t}{k} (\frac{\chi}{n})^k \\
    &\quad \times \binom{X_t}{k}(\frac{\chi}{n})^k (1-\frac{\chi}{n})^{n-2k}\\ \intertext{Using $\binom{n_1}{k_1}\binom{n_2}{k_2} \leq \binom{n_1+n_2}{k_1+k_2} $ gives}
    & \leq \sum_{k=0}^{\min \{n-X_t, X_t\}}\binom{n}{2k} (\frac{\chi}{n})^{2k} \cdot (1-\frac{\chi}{n})^{n-2k} \\ \intertext{Note that $\min \{n-X_t, X_t\} \leq \frac{n}{2}$}
    & \leq \sum_{k=0}^{\floor{\frac{n}{2}}}\binom{n}{2k} (\frac{\chi}{n})^{2k} \cdot (1-\frac{\chi}{n})^{n-2k} \\ \intertext{Consider a binomial random variable $Z\sim Bin(n,\frac{\chi}{n})$ and we consider the probability that $Z$ is even. Then, we have}
    & = \Pr\left(\text{$Z$ is even} \right) \\ \intertext{Using Lemma~\ref{lem:binomial} with $p=\frac{\chi}{n}$ gives}
    &=\frac{1}{2}+\frac{1}{2}\cdot (1-\frac{2\chi}{n})^n \\
    &\leq  \frac{1}{2} + \frac{1}{2e^{2\chi}} 
\end{align*}
Since $\chi>0$ is some constant, so $ \frac{1}{2} + \frac{1}{2e^{2\chi}}=c<1$ is also constant. Given that there exists $k>0$ s.t. $x_t^{(k)}=x_t$ w.h.p., we have
\begin{align*}
    \Pr(X_t^{(1)}=X_t^{(\lambda)}) &\leq  \Pr(\text{$\forall j \in [\lambda], $} X_t^{(j)}=X_t ) \\
    &\quad \times (1-e^{-\Omega(\lambda)})+ e^{-\Omega(\lambda)})\cdot 1 \\
                                  &\leq c^{\lambda} (1-e^{-\Omega(\lambda)})+ e^{-\Omega(\lambda)}  \\
                                  &= e^{-\Omega(\lambda)} .
\end{align*}
Then, $X_t^{(1)}>X_t^{(\lambda)}$ w.h.p and so is $Y_t^{(1)}>Y_t^{(\lambda)}$. Then we have $\lambda$ offspring cross the diagonal w.h.p if $X_t^{(\lambda)}\leq Y_t^{(1)}$.
\qed
\end{proof}

\SecFivePhaseTwoOnelemThree*

\begin{proof}[Proof of Lemma~\ref{lem:MGF1}]Compute for any $\eta> 0$ by using MGFs of binomial random variables:
\begin{align*}
 M_U(\eta)&= E[e^{\eta U}] \\
          &= E[e^{\eta (V_1-V_2)}] \\ \intertext{Using the independence of $V_1$ and $V_2$ gives}
          &= E[e^{\eta V_1}]E[e^{ -\eta V_2}] \\
          &= M_{V_1}(\eta) M_{V_2}(-\eta) \\ \intertext{Using the MGF of Binomial random variable gives}
          &=\left(1-\frac{\chi}{n}+\frac{\chi}{n}e^{\eta} \right)^{n-s} \left(1-\frac{\chi}{n}+\frac{\chi}{n}e^{-\eta} \right)^{s} \\ 
                           &= \left(1+\frac{\chi(e^{\eta}-1)}{n} \right)^{n-s}\left(1+\frac{\chi(e^{-\eta}-1)}{n} \right)^{s} \\ \intertext{Using $1+\frac{x}{n} \leq e^{x/n}$ for all $|x|<n$ gives the following.}
                           &\leq \exp \left(\frac{n-s}{n}\chi(e^{\eta}-1)+\frac{s}{n}\chi(e^{-\eta}-1)  \right)\\
                           &=\exp \left(\chi(e^{\eta}-1)+\frac{s}{n}\chi(e^{-\eta}-e^{\eta})  \right) \\ \intertext{Note that $e^{-\eta}-e^{\eta} <0$ for any $\eta>0$.} 
                           &\leq \exp \left(\chi(e^{\eta}-1)\right) 
\end{align*}
\qed
\end{proof}

\SecFivePhaseTwoOnelemFour*
\begin{proof}[Proof of Lemma~\ref{lem:MaxPoisson}]
    We apply Lemma~\ref{lem:MGF1} and Markov's inequality for the concentration. For any $s \geq 0$ and $\lambda \geq 1$,
    \begin{align*}
        \Pr \left(U \geq s \right)=\Pr \left(e^{\eta  U } \geq e^{\eta s}  \right) &\leq  E [e^{\eta U } ]e^{-\eta s}  \\ \intertext{Taking $E [e^{\eta  U } ] \leq \exp \left(\chi(e^{\eta}-1)\right)$ and $\eta:=\ln \ln \lambda$ gives}
                                  &\leq e^{\chi (\ln \lambda -1)} e^{- s\ln \ln \lambda } \\
                                  &\leq e^{-\chi} \lambda ^{\chi} e^{- s\ln \ln \lambda } 
    \end{align*}
For the second part, we notice that 
    \begin{align*}
        \Pr \left(U  \geq s \right)=\Pr \left(e^{\eta  U } \geq e^{\eta s}  \right) 
        &\leq  E [e^{\eta U } ]e^{-\eta s}  \\
        &\leq e^{-\chi} e^{\chi e^{\eta}-s\eta} \\ \intertext{Note that $h(\eta):=\chi e^{\eta}-s\eta$ attains its minimum at $\eta=\ln (\frac{s}{\chi})$ by differentiation. So setting $\eta:=\ln (\frac{s}{\chi})$ gives }
        &\leq  e^{-\chi}e^{-s \left(\ln (\frac{s}{\chi})-1 \right)} \\ \intertext{$s\geq e^2 \chi$ gives $\ln (\frac{s}{\chi})-1\geq 1$. Then, we have}
        &\leq e^{-\chi}e^{-s}.
    \end{align*}
    \qed
\end{proof}

\SecFivePhaseTwoTwolemOne*
\begin{proof}[Proof of Lemma~\ref{lem:crossD1}]
Let us divide the analysis into two cases. Assume $t$ is even and $Y_t>X_t$ first. The only way the search point can cross the diagonal is to cross horizontally. Assume now we are in iteration $t$ such that we mutate and select the new $X_{t+1}$. By using Lemma~\ref{lem:crossD0}, we conclude that if $t$ is even, then $\Pr \left( \exists k \in [\lambda] , x_t^{k}=x_t\right) \geq 1-e^{-\Omega(\lambda)}$ and $\Pr \left( \max_{i \in [\lambda]}X_t^i >\min_{i \in [\lambda]}X_t^i\right) \geq 1-e^{-\Omega(\lambda)}$.
Then, we first compute the probability that $E$ does not occur under the case that we have both $F_1:=\{\exists k \in [\lambda] , x_t^{k}=x_t\}$ and $F_2:=\{\max_{i \in [\lambda]}X_t^i >\min_{i \in [\lambda]}X_t^i \}$ hold.
\begin{align*}
    \Pr \left(E_t^c \mid F_1 \cap F_2 \right) 
    &= \Pr \left( X_t^{(1)} <Y_t\right) \\
    &= \Pr \left( \max_{i \in [\lambda]}X_t^i <Y_t\right) \\ \intertext{We define for $i \in          [\lambda]$ for $x_t$-bitstring, a random variable 
                    $W_{i,-}$ to be the number of 1-bits that the mutation operator flips into 0-bits and another 
                    the random variable $W_{i,+}$ to be the number of 0-bits that the mutation operator flips into 1-bits. Then, we define the change in 1-bits  $W_i=W_{i,+}-W_{i,-}$ is subject to the distribution $ \text{Bin}(n-X_t,\frac{\chi}{n})-\text{Bin}(X_t,\frac{\chi}{n})$. After the mutation in Algorithm~\ref{alg:ocoea2}, we have $\lambda $ offspring at iteration $t+1$:
                    $X_t^{i}=X_t+W_i$. So we have}
   \Pr \left(E_t^c  \mid F_1 \cap F_2 \right)&=\Pr \left(\max_{i \in [\lambda]}W_i<Y_t-X_t \right) \\
                          &= \prod_{i \in [\lambda]}\Pr \left(W_i<D_t \right) \\
                          &=\prod_{i \in [\lambda]} \left(1-\Pr \left(W_i \geq D_t \right) \right) \\ \intertext{ We consider the event that we only flip $0$-bit in $x_t$. We modify the probability pessimistically by assuming we flip exactly $ D_t $ bits in $0$-bit of $x_t$. In particular, $\Pr \left(W_i \geq D_t \right) \geq \Pr \left(W_{i,+}=D_t \right)$. Let $D_t=Y_t-X_t$.  Then, we  have }
                        &\leq \left(1- \binom{n-X_t}{D_t} \left(\frac{\chi}{n} \right)^{D_t} \right.\\
                        & \left. \quad \quad 
                        \times \left(1-\frac{\chi}{n} \right)^{n-D_t} \right)^{\lambda} \\ \intertext{Using $\binom{n-X_t}{D_t}\geq \frac{(n-X_t)^{D_t}}{D_t^{D_t}}$ and $n-X_t \geq \varepsilon n$ gives} 
                        & \leq \left(1-  \frac{(n-X_t)^{D_t}}{D_t^{D_t}} \left(\frac{\chi}{n} \right)^{D_t} \right. \\
                        &\left. \quad \quad \times \left(1-\frac{\chi}{n} \right)^{n-D_t} \right)^{\lambda} \\
                        &\leq \left(1-  \left( \frac{\chi \varepsilon}{D_t} \right)^{D_t}\left(1-\frac{\chi}{n} \right)^{n-D_t} \right)^{\lambda} \\ \intertext{Using $(1-\frac{x}{n})^n \geq e^{-x}(1-\frac{x^2}{n})$ and $1-\frac{x^2}{n}\geq \frac{1}{2}$ for sufficiently large $n$ gives}
                        &\leq  \left(1-\left( \frac{\chi \varepsilon}{D_t} \right)^{D_t} e^{-\chi \frac{n-D_t}{n}}\frac{1}{2}\right)^{\lambda} \\ \intertext{Recall that $D_t\leq c \leq n$. Thus, $\frac{n-D_t}{n}\leq 1 $ for sufficiently large $n$.}
                        &\leq  \left(1-\frac{\frac{1}{2}e^{-\chi}}{\left( \frac{D_t}{\varepsilon \chi}\right)^{D_t}}\right)^{\lambda} \\ \intertext{Using $(1-\frac{x}{n})^n \leq e^{-x}$ gives}
                        &\leq \exp \left(-\frac{1}{2e^{\chi}} \frac{\lambda}{\left( \frac{D_t}{\varepsilon \chi}\right)^{D_t}} \right) \\ \intertext{If $\left(\frac{D_t}{\varepsilon \chi} \right)^{D_t} \leq \frac{\lambda}{ \ln \lambda}$, this gives }
                        &\leq \exp \left(-\frac{1}{2e^{\chi}} \frac{\lambda}{\lambda / \ln \lambda} \right) \\
                        &=\left(\frac{1}{\lambda} \right)^{\frac{1}{2e^{\chi}}}
\end{align*}

So we would like to check the range of  $D_t$ which satisfies $\left(\frac{D_t}{\varepsilon \chi} \right)^{D_t} \leq \frac{\lambda}{ \ln \lambda}$. It is known that the Lambert function $W(x)=\ln x - \ln \ln x +o(1)$ \cite{corless1996lambert,LambertWfunction1}. We can derive $W(x) \leq  \ln (x)$, and thus we have 
\begin{align*}
  \frac{ \ln \left(\frac{\lambda}{ \ln \lambda} \right)}{W \left(\frac{ \ln \left(\frac{\lambda}{ \ln \lambda} \right)}{\varepsilon \chi}\right)} &\geq \frac{ \ln \left(\frac{\lambda}{ \ln \lambda} \right)}{ \ln\left(\frac{ \ln \left(\frac{\lambda}{ \ln \lambda} \right)}{\varepsilon \chi}\right)}  \\
  &= \frac{\ln \lambda - \ln \ln \lambda}{\ln \ln \left(\frac{\lambda}{\ln \lambda} \right) +\ln (\frac{1}{\varepsilon \chi})
    } \\
     &= \frac{\ln \lambda - \ln \ln \lambda}{\ln \ln \lambda - \ln \ln \ln \lambda  +\ln (\frac{1}{\varepsilon \chi})
    } \\
    &\geq c = \frac{\kappa \ln \lambda }{\ln \ln \lambda} \\ \intertext{for any constant $\kappa \in (0,1)$.}
\end{align*}
So $D_t \leq c$ implies that $D_t \leq \frac{ \ln \left(\frac{\lambda}{ \ln \lambda} \right)}{W \left(\frac{ \ln \left(\frac{\lambda}{ \ln \lambda} \right)}{\varepsilon \chi}\right)}$. Next, if $D_t \leq \frac{ \ln \left(\frac{\lambda}{ \ln \lambda} \right)}{W \left(\frac{ \ln \left(\frac{\lambda}{ \ln \lambda} \right)}{\varepsilon \chi}\right)}$, then by definition of $W$ function ($e^{W(x)}=\frac{x}{W(x)}$), we have

\begin{align*}
     \frac{D_t}{\varepsilon \chi} &\leq \frac{\frac{ \ln \left(\frac{\lambda}{ \ln \lambda} \right)}{\varepsilon \chi}}{ W \left(\frac{ \ln \left(\frac{\lambda}{ \ln \lambda} \right)}{\varepsilon \chi}\right)} \\
     &= \exp \left( W\left( \frac{ \ln \left(\frac{\lambda}{ \ln \lambda} \right)}{\varepsilon \chi} \right) \right)  \\ \intertext{Taking $\ln$ gives}
\ln \left(\frac{D_t}{\varepsilon \chi} \right) &  \leq  W\left( \frac{ \ln \left(\frac{\lambda}{ \ln \lambda} \right)}{\varepsilon \chi} \right)  \\ \intertext{Considering $y\leq W(x)$ is the solution to the inequality $ye^y\leq x$ and using $x=\frac{ \ln \left(\frac{\lambda}{ \ln \lambda} \right)}{\varepsilon \chi}, y=\ln \left(\frac{D_t}{\varepsilon \chi} \right)$ give}
     \ln \left(\frac{D_t}{\varepsilon \chi} \right) e^{\ln \left(\frac{D_t}{\varepsilon \chi} \right)}                     & \leq \frac{\ln \left(\frac{\lambda}{ \ln \lambda} \right)}{\varepsilon \chi} \\ \intertext{Taking exponent on both sides and rearranging the expression give}
    \left(\frac{D_t}{\varepsilon \chi} \right)^{D_t} &\leq \frac{\lambda}{ \ln \lambda} .
\end{align*}
Now, we combine the probability of $E_t$ under the case that we have both $F_1=\{\exists k \in [\lambda] , x_t^{k}=x_t\}$ and $F_2=\{\max_{i \in [\lambda]}X_t^i >\min_{i \in [\lambda]}X_t^i \}$ hold with Lemma~\ref{lem:crossD0} to give the overall probability of $E_t$. Using the law of total probability gives:
\begin{align*}
    \Pr\left(E_t^c \right) & =  \Pr\left(E_t^c \mid F_1 \cap F_2 \right) \Pr \left(F_1 \cap F_2 \right) \\
                           & \quad +\Pr\left(E_t^c \mid F_1^c \cup F_2^c \right) \Pr \left(F_1^c \cup F_2^c  \right) \\
                           &\leq \left(\frac{1}{\lambda} \right)^{\frac{1}{2e^{\chi}}} +\Pr \left(F_1^c \cup F_2^c  \right) \\ \intertext{Using Union bound and Lemma~\ref{lem:crossD0} gives} 
                           &\leq \left(\frac{1}{\lambda} \right)^{\frac{1}{2e^{\chi}}} + 2 e^{-\Omega(\lambda)} \\ \intertext{For sufficiently large $\lambda$, we have $2 e^{-\Omega(\lambda)} \leq  \left(\frac{1}{\lambda} \right)^{\frac{1}{2e^{\chi}}} $. So, this leads to}
                           &\leq 2\left(\frac{1}{\lambda} \right)^{\frac{1}{2e^{\chi}}}
\end{align*}
\qed
\end{proof}

\SecFivePhaseTwoTwolemTwo*

\begin{proof}[Proof of Lemma~\ref{lem:crossD2}]
A similar statement works for $\Delta Y_t^i$. In this proof, we deal with the case $\Delta X_t^i$. We assume $t$ is even and $X_t<Y_t$.
Statement (A) follows immediately from Lemma~\ref{lem:crossD1} because any constant $\kappa\in(\chi,(1+\chi)/2)\subset (0,1)$ for $\chi \in (0,1)$.
\par Next, we deal with the second inequality.
To estimate $K$, we use Lemma~\ref{lem:MaxPoisson}.
We denote the number of $1$-bits that $i$-th offspring increases by $U_i$ for $i \in [\lambda]$. Lemma~\ref{lem:MaxPoisson} implies that for all $i \in [\lambda]$,
\begin{align*}
    \Pr \left(U_i \geq D_t +c \right) 
    &\leq e^{-\chi} \lambda ^{\chi}e^{-(c+D_t) \ln \ln \lambda} \\
    &=e^{-\chi}\lambda^{\chi-\kappa}e^{-D_t \ln \ln \lambda}:=q_k.
\end{align*}
We can consider such sampling with at least $(D_t+c)$-jump as $\lambda$ independent trials, and thus $K$ is
stochastic dominated by a binomial distributed random variable $Bin(\lambda,q_K) \sim K^*$. So 
\begin{align*}
    E[K^*]=\lambda q_K &=e^{-\chi}\lambda^{1+\chi-\kappa}e^{-D_t \ln \ln \lambda}\\ \intertext{Using $D_t\leq \frac{\kappa \ln \lambda}{\ln \ln \lambda}$ gives}
                     &\geq e^{-\chi} \lambda^{1+\chi -2\kappa}
\end{align*}        
 We can apply the Chernoff's bound to get for any $\delta_1 \in (0,1)$, 
\begin{align*}
    &\Pr \left(K > (1+\delta_1) E[K^*] \right) \\ \intertext{Using the definition of stochastic dominance $(K^*\succcurlyeq K)$ gives}
    \leq &\Pr \left(K^* > (1+\delta_1) E[K^*] \right) \\  \intertext{Using Chernoff's bound gives}
    \leq& e^{-E[K^*]\delta_1^2 /3} \\ \intertext{Using the lower bound for $\E{K^*}$ gives}
     \leq &e^{\frac{e^{-\chi} \lambda^{1+\chi -2\kappa}\delta_1^2}{3}}=\exp(-\lambda \Omega(1))
\end{align*}
where $1+\chi-2\kappa>0$ directly follows from $\kappa \in (\chi,\frac{1+\chi}{2})$. 
\par To estimate $M$, we note that 
\begin{align*}
 \Pr \left( \Delta X_t^i \geq D_t\right) 
    &\geq \binom{n-X_t}{D_t} (\frac{\chi}{n})^{D_t}(1-\frac{\chi}{n})^{n-D_t} \\  \intertext{Using $\binom{n-X_t}{D_t}\geq \frac{(n-X_t)^{D_t}}{D_t^{D_t}}$ and $n-X_t \geq \varepsilon n$ gives}
    &\geq \frac{(n-X_t)^{D_t}}{D_t^{D_t}} \left(\frac{\chi}{n} \right)^{D_t} (1-\frac{\chi}{n})^{n-D_t} \\ \intertext{Using $n-X_t \geq \varepsilon n$ gives}
    &\geq \frac{(\varepsilon n)^{D_t}}{D_t^{D_t}} \left(\frac{\chi}{n} \right)^{D_t} (1-\frac{\chi}{n})^{n-D_t} \\ \intertext{Using $(1-\frac{x}{n})^n \geq e^{-x}(1-\frac{x^2}{n})$ and $1-\frac{x^2}{n}\geq \frac{1}{2}$ for sufficiently large $n$ gives}
    &\geq \frac{1}{2e^{\chi}} \left(\frac{\varepsilon \chi}{D_t} \right)^{D_t}:=q_M 
\end{align*}
$M$ is stochastic dominated the  random variable $M^* = Bin(\lambda, q_M)$
So we have 
\begin{align*}
    E[M^*]&=\lambda q_M  \\ \intertext{$D_t \leq c$ implies that }
          &\geq \frac{1}{2e^{\chi}} \lambda \left(\frac{\varepsilon \chi}{c} \right)^c \\ \intertext{Note that $c=\frac{\kappa \ln \lambda}{\ln \ln \lambda}$ and thus $\left(\frac{\varepsilon \chi}{c} \right)^c=
          \frac{1}{e^{\frac{\kappa \ln \lambda}{\ln \ln \lambda}\ln \left(\frac{\kappa \ln \lambda}{\varepsilon \chi \ln \ln \lambda} \right) } } \geq \lambda^{-\kappa} $. Then, we have}
          &\geq \frac{\lambda^{1-\kappa}}{2e^{\chi}}
\end{align*}
where $\kappa<1$ follows from $\kappa \in (\chi,\frac{1+\chi}{2})$.
 By using Chernoff's bound on $M^*$ and $M^*\preccurlyeq  M$, we have for any $\delta_2 \in (0,1)$, 
\begin{align*}
    \Pr \left( M \leq (1-\delta_2) \lambda q_M \right) 
    &\leq \Pr \left(M^* \leq (1-\delta_2) \lambda q_M \right) \\
    &\leq e^{-E[M^*]\delta_2^2 /2} \\ \intertext{Using the lower bound for $\E{M^*}$ gives}
    &\leq e^{-\lambda^{1-\kappa}\delta_2^2 /2e^{\chi}}
\end{align*}
Now, we have for any $\delta_1, \delta_2 \in (0,1)$, 
\begin{align*}
    \Pr \left( K> (1+\delta_1) \lambda q_K \right) &\leq e^{\frac{e^{-\chi} \lambda^{1+\chi -2\kappa}\delta_1^2}{3}} \\
    \Pr \left( M \leq (1-\delta_2) \lambda q_M \right) &\leq e^{-\lambda^{1-\kappa}\delta_2^2 /2e^{\chi}}.
\end{align*}
We then use the union bound to prove that either of $F_1=\{ K> (1+\delta_1) \lambda q_K\}$ and $F_2=\{M \leq (1-\delta_2) \lambda q_M\}$ occurs with exponentially small probability in $\lambda$:
\begin{align*}
    \Pr \left(F_1 \cup F_2 \right) \leq e^{\frac{e^{-\chi} \lambda^{1+\chi -2\kappa}\delta_1^2}{3}}+ e^{-\lambda^{1-\kappa}\delta_2^2 /2e^{\chi}}.
\end{align*}
So we can prove that 
\begin{align*}
    \Pr \left(\frac{K}{M} \leq \frac{1+\delta_1}{1-\delta_2} \frac{q_K \lambda}{q_M \lambda} \right) 
    &\geq  \Pr \left(F_1^c \cap F_2^c \right) \\
            & = 1 -\Pr \left(F_1 \cup F_2 \right) \\
            &\geq 1 - e^{\frac{e^{-\chi} \lambda^{1+\chi -2\kappa}\delta_1^2}{3}} \\
            &\quad \quad - e^{-\lambda^{1-\kappa}\delta_2^2 /2e^{\chi}} \\
            &=1-e^{-\Omega (\lambda)}
\end{align*}
Finally, we want to bound the ratio $\frac{q_K}{q_M}$ in a small quantity. Note that
\begin{align*}
   \frac{q_K}{q_M} 
   &\leq \frac{e^{-\chi}\lambda^{\chi-\kappa}e^{-D_t \ln \ln \lambda}}{\frac{1}{2e^{\chi}} \left(\frac{\varepsilon \chi}{D_t} \right)^{D_t}}  \\
   &= 2 \lambda^{\chi-\kappa} \frac{1}{e^{D_t \ln \ln \lambda} \left(\frac{\varepsilon \chi}{D_t} \right)^{D_t}} \\ \intertext{Note that for $y \in [1,c]$, $\phi(y):=e^{y \ln \ln \lambda} \left(\frac{\varepsilon \chi}{y} \right)^{y}$ attains its minimum at $1$ or $c$ from extreme value theorem and the fact that $\phi'(y)<0$ for $y \in [1,c]$. So $\phi(D_t) \geq \min \{\phi(1),\phi(c)\}$. We can see $\phi(1)=e^{\ln \ln \lambda}$ and $\phi(c)=\lambda^{\kappa} \cdot (\frac{\varepsilon \chi}{c})^c \geq \lambda^{\kappa} \cdot \lambda^{-\kappa}=1$. Then, $\phi(D_t) \geq 1$.}
   &\leq 2 \lambda^{\chi-\kappa}.
\end{align*}
In overall, for any $\chi \in (0,1)$, $\kappa \in (\chi,\frac{\chi+1}{2}), \delta_1, \delta_2 \in (0,1)$, we derive
\begin{align*}
\Pr \left(\frac{K}{M} \leq 2\frac{1+\delta_1}{1-\delta_2}  \left(\frac{1}{\lambda} \right)^{\kappa -\chi} \right) \geq 1- e^{-\Omega(\lambda)}.
\end{align*}
Note that $\frac{1+\delta_1}{1-\delta_2}= 1 + \frac{\delta_1+\delta_2}{1-\delta_2}= 1+ \delta$ where $\delta:=\frac{\delta_1+\delta_2}{1-\delta_2} >0$.
The second case follows from a similar analysis and 
we complete the proof.
\qed
\end{proof}

\SecFivePhaseTwoTwolemThree*

\begin{proof}[Proof of Lemma~\ref{lem:crossD3}]
In this proof, we assume that $Y_t>X_t$ and a similar analysis is applied to the case $Y_t\leq X_t$. Let us choose $c$ in Lemma~\ref{lem:crossD2}. For any $\chi \in (0,1)$ and $c$ ( $c= \frac{\kappa \ln \lambda}{\ln \ln \lambda}$ where any constant $\kappa \in (\chi, \frac{1+\chi}{2})$), we have 
\begin{itemize}
    \item[(A):] \text{$\Pr\left(\max_{i \in [\lambda]} \Delta X_t^i \geq  D_t \mid D_t< c \right) \geq 1-2\left(\frac{1}{\lambda} \right)^{\frac{1}{2e^{\chi}}}$}
    \item[(B):] \text{$\Pr \left(\frac{K}{M} \leq 2\frac{1+\delta_1}{1-\delta_2}  \left(\frac{1}{\lambda} \right)^{\kappa -\chi} \right) \geq 1- e^{-\Omega(\lambda)}$ for any } \text{constant} $\delta_1,\delta_2 \in (0,1)$
\end{itemize}

where $\Delta X_t^i$ denotes the change of number of $1$-bit in $i$-th offspring of $x_t$, and 
\begin{align*}
    K:&=\{\text{\# of samples/offspring s.t. $\Delta X_t^i \geq D_t+ c$}\}\\
    M:&=\{\text{\# of samples/offspring s.t. $\Delta X_t^i \geq D_t$}\}.
\end{align*}

Given that $D_t \in [1,c)$, some of the offspring may escape from the lower boundary $(Y_t=X_t-c)$ of the tube, while there is only a small probability that the algorithm selects such escaping offspring. By using Lemma~\ref{lem:crossD2}, we have
    \begin{align*}
        &\Pr \left(D_{t+1}>c \mid D_t<c \right)  \\
       =&\Pr \left(|X_{t+1}-Y_{t+1}|>c \mid D_t<c \right) \\ \intertext{Let $E=$\{$\lambda$ offspring cross the diagonal\} and we use Law of total probability}
       =&\Pr \left(\Delta X_t-D_t>c  \mid  D_t<c, E \right) \Pr (E \mid D_t<c) \\
        & + \Pr \left(D_t-\Delta X_t>c \mid D_t<c, E^c\right) \Pr(E^c \mid D_t<c) \\ \intertext{Using definition of conditional probability and the condition of $c$ from Lemma~\ref{lem:crossD2} gives}
        \leq & \Pr \left(\Delta X_t-D_t>c  \mid D_t<c \right)  \\
            & + 2\Pr \left(D_t-\Delta X_t>c \mid D_t<c, E^c\right) \left(\frac{1}{\lambda} \right)^{\frac{1}{2e^{\chi}}} \\ 
        \leq & \Pr \left(\Delta X_t>D_t+c  \mid D_t<c \right) + 2\cdot  \left(\frac{1}{\lambda} \right)^{\frac{1}{2e^{\chi}}}   \\ \intertext{We define $K:=\{\text{\# of samples/offspring s.t. $\Delta X_t^i \geq D_t+c$}\}$,
        \newline $M:=\{\text{\# of samples/offspring s.t. $\Delta X_t^i \geq D_t$}\}$. The algorithm assigns the fitness $1$ to $M$ offspring in this case including those crossing outside the tube (denote their number by $K$). Since their fitness is the same, the next search point will be selected uniformly at random from $M$. Using the second condition in Lemma~\ref{lem:crossD2} the and Law of total probability again, we have for any constant $\delta_1,\delta_2 \in (0,1)$,}     
        \leq & \Pr \left(\Delta X_t>D_t+c  \mid D_t<c,  \frac{K}{M}< 2\frac{1+\delta_1}{1-\delta_2}  \left(\frac{1}{\lambda} \right)^{\kappa -\chi} \right) \\
        &\quad \times \Pr \left(  \frac{K}{M} <2\frac{1+\delta_1}{1-\delta_2}  \left(\frac{1}{\lambda} \right)^{\kappa -\chi} \right) \\
        & + \Pr \left(\Delta X_t>D_t+c  \mid D_t<c,  \frac{K}{M}\geq 2\frac{1+\delta_1}{1-\delta_2}  \left(\frac{1}{\lambda} \right)^{\kappa -\chi} \right) \\
        &\quad \times \Pr \left(  \frac{K}{M} \geq 2\frac{1+\delta_1}{1-\delta_2}  \left(\frac{1}{\lambda} \right)^{\kappa -\chi} \right) \\
        & + 2\cdot  \left(\frac{1}{\lambda} \right)^{\frac{1}{2e^{\chi}}}   \\ \intertext{Using (B) in Lemma~\ref{lem:crossD2} gives}
        \leq & 2\frac{1+\delta_1}{1-\delta_2}  \left(\frac{1}{\lambda} \right)^{\kappa -\chi}\cdot 1 + \Pr \left(  \frac{K}{M} \geq 2\frac{1+\delta_1}{1-\delta_2}  \left(\frac{1}{\lambda} \right)^{\kappa -\chi} \right) \\
        &+  2\left(\frac{1}{\lambda} \right)^{\frac{1}{2e^{\chi}}} \\
       \leq &2\frac{1+\delta_1}{1-\delta_2}  \left(\frac{1}{\lambda} \right)^{\kappa -\chi}+ e^{-\Omega(\lambda)}+ 2\left(\frac{1}{\lambda} \right)^{\frac{1}{2e^{\chi}}} \\ \intertext{Taking $\delta_1=\delta_2=\frac{1}{2}$ gives}
      \leq & 6 \left(\frac{1}{\lambda} \right)^{\min\{\kappa-\chi, \frac{1}{2e^{\chi}} \}}+ \left(\frac{1}{\lambda} \right)^{\min\{\kappa-\chi, \frac{1}{2e^{\chi}} \}}+ 2 \left(\frac{1}{\lambda} \right)^{\min\{\kappa-\chi, \frac{1}{2e^{\chi}} \}} \\ \intertext{Notice that $\kappa \in (\chi,\frac{\chi+1}{2})$ and thus $\kappa-\chi \in (0,\frac{1-\chi}{2})$. It is known that $e^x \geq 1+x$ for all $x\in \mathbb{R}$. In our case, we derive $\frac{1-\chi}{2} \leq \frac{1}{2e^\chi}$ for any $\chi \in (0,1)$. Then, $\min\{\kappa-\chi, \frac{1}{2e^{\chi}} \}=\kappa-\chi$. Then, we derive}
      \leq & 9 \left(\frac{1}{\lambda} \right)^{\kappa-\chi} \\ \intertext{We introduce another variable $\gamma:=\kappa-\chi$ and we complete the proof.}
      =&9 \left(\frac{1}{\lambda} \right)^{\gamma}.
    \end{align*}
\qed
\end{proof}

\SecFivePhaseTwoTwolemFour*

\begin{proof}[Proof of Lemma~\ref{lem:crossD4}]
We use Lemma~\ref{lem:crossD3} can conclude that with probability at most $ 9 \left(\frac{1}{\lambda} \right)^{\gamma }$ where constant $\gamma \in (0,\frac{1-\chi}{2})$, the search point deviates from $c$-tube where $c$ in Lemma~\ref{lem:crossD2}. By Lemma~\ref{lem:crossD1}, the offspring crosses the diagonal with probability at least $1-\left(\frac{1}{\lambda} \right)^{\frac{1}{2e^{\chi}}}$. So we have shown the first part and second parts directly from the fact that at each iteration, when the search point stays within the tube, either $X_{t+1}-X_t\geq 1$ or $Y_{t+1}-Y_t \geq 1$ with probability at least $1-9 \left(\frac{1}{\lambda} \right)^{\gamma }-\left(\frac{1}{\lambda} \right)^{\frac{1}{2e^{\chi}}}=1-O(\frac{1}{\lambda ^\gamma})$. By taking $\gamma:= \frac{1-\chi}{4}$ (and thus in Lemma~\ref{lem:crossD3}, $\kappa=\frac{1+3\chi}{4}$), 
We conclude that either $X_{t+1}-X_t\geq 1$ or $Y_{t+1}-Y_t \geq 1$ with high probability i.e. with probability $1-O(\frac{1}{\lambda ^{(1-\chi)/4}})$. Finally, by taking sufficiently large $\lambda$ s.t. $1-O(\frac{1}{\lambda ^{(1-\chi)/4}}) \geq \frac{1}{2}$, it yields the positive constant drift $ \frac{1}{2}$ for $H_t$.
\qed
\end{proof}

\SecFivePhaseTwoThreeMainOne*

\begin{proof}
[Proof of Theorem~\ref{thm:Main}]
Assume Algorithm~\ref{alg:ocoea2} starts from $(X_0,Y_0)=(0,0)$, $t$ is even if $X_t<Y_t$ and $t$ is odd if $X_t \geq Y_t$. We use $H_t=2n-X_t-Y_t$ as our potential function. We define $D_t=|X_t-Y_t|$ as the distance function away from the diagonal in $|x|_1$-$|y|_1$ plane. 
We define $T:=\inf \{t>0 \mid H_t \leq \delta n\}$ for any constant $\delta \in (0,1)$. Then we want to show $\E{T}\leq 2(2-\delta)n$.
Note that $c$ defined in  Lemma~\ref{lem:crossD4} and  
\begin{align*}
    T:&=\inf \{t>0 \mid H_t \leq \delta n\} \\
      &\leq \inf \{t>0 \mid D_t<c \wedge  H_t \leq \delta n\}. \\ \intertext{We replace $\delta$ with $3\varepsilon$ where $\varepsilon \in (0,1/3)$ and $c=\frac{1+3\chi}{4 \ln \ln \lambda} \ln \lambda$.}
      &=  \inf \{t>0 \mid D_t<c \wedge  H_t \leq 3\varepsilon n\}:= T_{c,3\varepsilon}.
\end{align*}
If $D_t <c$, without loss of generality assume $X_t\geq Y_t$, then $D_t <c$ implies that $n-Y_t\geq n-X_t, X_t-Y_t<c$ and moreover $H_t>3\varepsilon n$ implies that $n-X_t>3\varepsilon n-n+Y_t>3\varepsilon n-n+X_t-c$. This is equivalent to $n-X_t>\frac{3\varepsilon n-c}{2}$. Thus, if $D_t <c$ and $H_t>3\varepsilon n$, then 
\begin{align*}
    n-X_t > \frac{3\varepsilon n-c}{2} 
    &\And n-Y_t > \frac{3\varepsilon n-c}{2}.  \\ \intertext{ In Lemma~\ref{lem:crossD4}, we assume $c= \frac{\frac{1+3\chi}{4} \ln \lambda}{\ln \ln \lambda}=O(\ln\lambda)=o(n)$ where the last equality follows from the assumption in this paper $\lambda=\text{poly}(n)$. For sufficiently large $n$, we have for any $\varepsilon \in (0,1/3)$,}
    n-X_t > \varepsilon n 
    &\And n-Y_t > \varepsilon n .
\end{align*}
Now, we meet the conditions of Lemma~\ref{lem:crossD4}. 
Under this setting, by Lemma~\ref{lem:crossD4}, $\Et{H_{t}-H_{t+1}} \geq 1/2$. By additive drift theorem \cite{he_drift_2001}, 
\begin{align*}
    \E{T} \leq \E{T_{c,3\varepsilon}}\leq \frac{2n-3\varepsilon n}{1/2} = \frac{(2-3\varepsilon)n}{ 1/2} = 2(2-\delta)n.
\end{align*}
Now, consider the failure event that a cycle (consecutive 
steps) fails to be successful at some iteration $t$. We want to 
use Lemma~\ref{lem:phase1} to restart the analysis if such a 
failure occurs.
To clarify our argument, recall from Definition 7 that a cycle is two
consecutive generations of the algorithm. A cycle can be
\emph{successful} (see Definition 7). We
call $T_c$ such cycles a \emph{$T_c$-cycle phase}, and we have a failure in a $T_c$-cycle phase if any of the $T_c$ cycles is unsuccessful.
Note that the event that a successful cycle in Definition~\ref{def:phase1} breaks is equivalent to the 
event that either the search point escapes from the $c$-tube or 
does not cross the diagonal during the cycle. By 
Lemma~\ref{lem:crossD1}, the probability that the search point 
does not cross the diagonal is at most $p_c = 2(\frac{1}
{\lambda})^{1/2e^{\chi}}$. By 
Lemma~\ref{lem:crossD3}, the probability that the search point lies outside the $c$-tube where $c$ defined in Lemma~\ref{lem:crossD3} is at most $p_e=9 (\frac{1}{\lambda})^{(1-\chi)/4}$. By Lemma~\ref{lem:phase1}, we have 
\begin{align*}
    &\Pr \left(\text{The algorithm has $T_c$ consecutive successful cycles} \right) \\
    \geq& 1 - 2T_c (p_c+p_e) \\ \intertext{Substituting $p_c = 2(\frac{1}
{\lambda})^{1/2e^{\chi}}$ and $p_e=9 (\frac{1}{\lambda})^{(1-\chi)/4}$ gives}
    \geq & 1 -2T_c \left(2(\frac{1}
{\lambda})^{1/2e^{\chi}}+ 9 (\frac{1}{\lambda})^{(1-\chi)/4} \right) \\ \intertext{Recall that $e^{-\chi} \geq \frac{1-\chi}{2} \geq \frac{1-\chi}{4}$ for $\chi \in (0,1)$. Then, we have}
    \geq & 1 - 2T_c \cdot 11 (\frac{1}{\lambda})^{(1-\chi)/4}  \\ \intertext{Picking $T_c=4(2-\delta)n$ and $\lambda \geq (264(2-\delta)n)^{4/(1-\chi)}$ gives}
    \geq & 1- 4(2-\delta)n \cdot 22 \cdot \frac{1}{264(2-\delta)n} = \frac{2}{3}.
\end{align*}
 This means that the algorithm has $4(2-\delta)n$ successful
consecutive cycles with probability at least $\frac{2}{3}$. If the algorithm fails to have such a successful
consecutive cycles, then we restart the algorithm.
In a successful phase, the algorithm finds 
a $\delta$-approximation in expected time $2(2-\delta)n$ generations. By Markov's inequality, within $4(2-\delta)n$ generations with probability at least $1/2$ (otherwise we consider this a failure). By taking a union bound on the failure events (i.e. an
unsuccessful $T_c$-cycle phase or the algorithm does not find the $\delta$-approximation within the end of the phase, the algorithm finds the $\delta$-approximation in a phase of length $T_c$ cycles with probability $p \geq 1-1/3-1/2=1/6$. Hence, we multiply a factor of $6$ to obtain the overall expected runtime is at most $4 (2-\delta)n/p \leq 24(2-\delta)n$.
\par   Notice that we only compute the number of generations so far. To obtain the algorithm's runtime (number of function evaluations), we need to multiply $\lambda$ for each iteration. So we end up with the expected runtime of Algorithm~\ref{alg:ocoea2} is $4  \lambda   (2-\delta)n$ for $\lambda = \Omega(n^{\frac{4}{1-\chi}})$.
\qed
\end{proof}

\end{document}